\documentclass[lettersize,journal]{IEEEtran}
\usepackage{amsmath,amsfonts}
\usepackage{algorithmic}
\usepackage{array}
\usepackage[caption=false,font=normalsize,labelfont=sf,textfont=sf]{subfig}
\usepackage{textcomp}
\usepackage{stfloats}
\usepackage{url}
\usepackage{verbatim}
\usepackage{graphicx}
\def\BibTeX{{\rm B\kern-.05em{\sc i\kern-.025em b}\kern-.08em
    T\kern-.1667em\lower.7ex\hbox{E}\kern-.125emX}}
\usepackage{balance}

\usepackage[ruled,linesnumbered]{algorithm2e}
\usepackage{ntheorem}
\usepackage{amssymb}
\usepackage{booktabs}
\usepackage{multirow}
\usepackage{ragged2e}

\graphicspath{{figs/}}
\DeclareMathOperator*{\argmax}{arg\,max}

\newtheorem{theorem}{Theorem}
\newtheorem{proof}{Proof}

\begin{document}

\title{Data-driven Knowledge Fusion for Deep Multi-instance Learning}

\author{
Yu-Xuan~Zhang,~\IEEEmembership{Student~Member,~IEEE},~Zhengchun~Zhou,~\IEEEmembership{Member,~IEEE},~Xingxing~He,\\
Avik~Ranjan~Adhikary,~\IEEEmembership{Member,~IEEE},~and~Bapi~Dutta
    \thanks{
    This work was supported in part by the National Natural Science Foundation of China (62131016).
    \emph{(Corresponding author: Zhengchun Zhou)}
    }
    \thanks{
    Y.-X. Zhang, Z. Zhou, and A. R. Adhikary are with the School of Information Science and Technology, Southwest Jiaotong University, Chengdu 611730, China.
    (e-mail: inki.yinji@gmail.com, zzc@swjtu.edu.cn, and avik.adhikary@ieee.org)
    }
    \thanks{
    X. He is with the School of Mathematics, Southwest Jiaotong University, Chengdu 611730, China.
    (e-mail: x.he@home.swjtu.edu.cn)
    }
    \thanks{
    B. Dutta is with the Department of Computer Science, University of Ja\'{e}n, Ja\'{e}n 23071, Spain.
    (e-mail: bdutta@ujaen.es)
    }
}

\maketitle

\begin{abstract}
Multi-instance learning (MIL) is a widely-applied technique in practical applications that involve complex data structures.
MIL can be broadly categorized into two types: traditional methods and those based on deep learning.
These approaches have yielded significant results, especially with regards to their problem-solving strategies and experimental validation, providing valuable insights for researchers in the MIL field.
However, a considerable amount of knowledge is often trapped within the algorithm, leading to subsequent MIL algorithms that solely rely on the model's data fitting to predict unlabeled samples.
This results in a significant loss of knowledge and impedes the development of more intelligent models.
In this paper, we propose a novel data-driven knowledge fusion for deep multi-instance learning (DKMIL) algorithm.
DKMIL adopts a completely different idea from existing deep MIL methods by analyzing the decision-making of key samples in the data set (referred to as the data-driven) and using the knowledge fusion module designed to extract valuable information from these samples to assist the model's training.
In other words, this module serves as a new interface between data and the model, providing strong scalability and enabling the use of prior knowledge from existing algorithms to enhance the learning ability of the model.
Furthermore, to adapt the downstream modules of the model to more knowledge-enriched features extracted from the data-driven knowledge fusion module, we propose a two-level attention module that gradually learns shallow- and deep-level features of the samples to achieve more effective classification.
We will prove the scalability of the knowledge fusion module while also verifying the efficacy of the proposed architecture by conducting experiments on 38 data sets across 6 categories.
\end{abstract}


\section{Introduction}

\IEEEPARstart{M}{achines} with the ability to think like humans have been the eternal pursuit of many artificial intelligence researchers.
Based on this, individuals from various fields are working tirelessly to contribute their best efforts, such as using evolutionary computation \cite{Back:1997:317,Dejong:2017:373388} to evolve finite state machines for events prediction based on the past observations;
using knowledge graph \cite{Chen:2020:112948,Hogan:2021:137} to identify errors and draw new conclusions from existing data;
using reinforcement learning \cite{Kaelbling:1996:237285,Arulkumaran:2017:2638} to train agents through trial-and-error interactions with the environment;
and using abductive learning \cite{Zhou:2019:13,Cai:2021:18151821} to combine machine learning with first-order logical reasoning.
In a similar vein, we hope that multi-instance learners can analyze existing methods and extract latent knowledge to enhance their models' capacity for learning.

Multi-instance learning (MIL) was developed by Dietterich et al. \cite{Dietterich:1997:3171} for drug activity prediction and was formulated as a paradigm for handling complex data structures.
In MIL, each data sample is called a bag with multiple instances.
And only bag labels are provided, whereas instance labels are either unavailable or nonexistent.
This kind of learning paradigm can better represent real-world phenomena, and up to this point, successful applications have been built in image classification \cite{Wu:2018:10651080,Yang:2021:54565467,Zeng:2022:110}, web page recommendation \cite{Zhou:2005:135147,Wei:2019:21092120,Huang:2022:108583}, video anomaly detection (VAD) \cite{Sultani:2018:64796488,Tian:2021:49754986,Li:2022:19}, and medical diagnosis \cite{Lin:2022:interventional,Zhu:2022:113,Shao:2021:21362147}, among others.

\begin{figure*}[!t]
    \centering
    \subfloat[Semantic example]{\includegraphics[width=0.48\hsize]{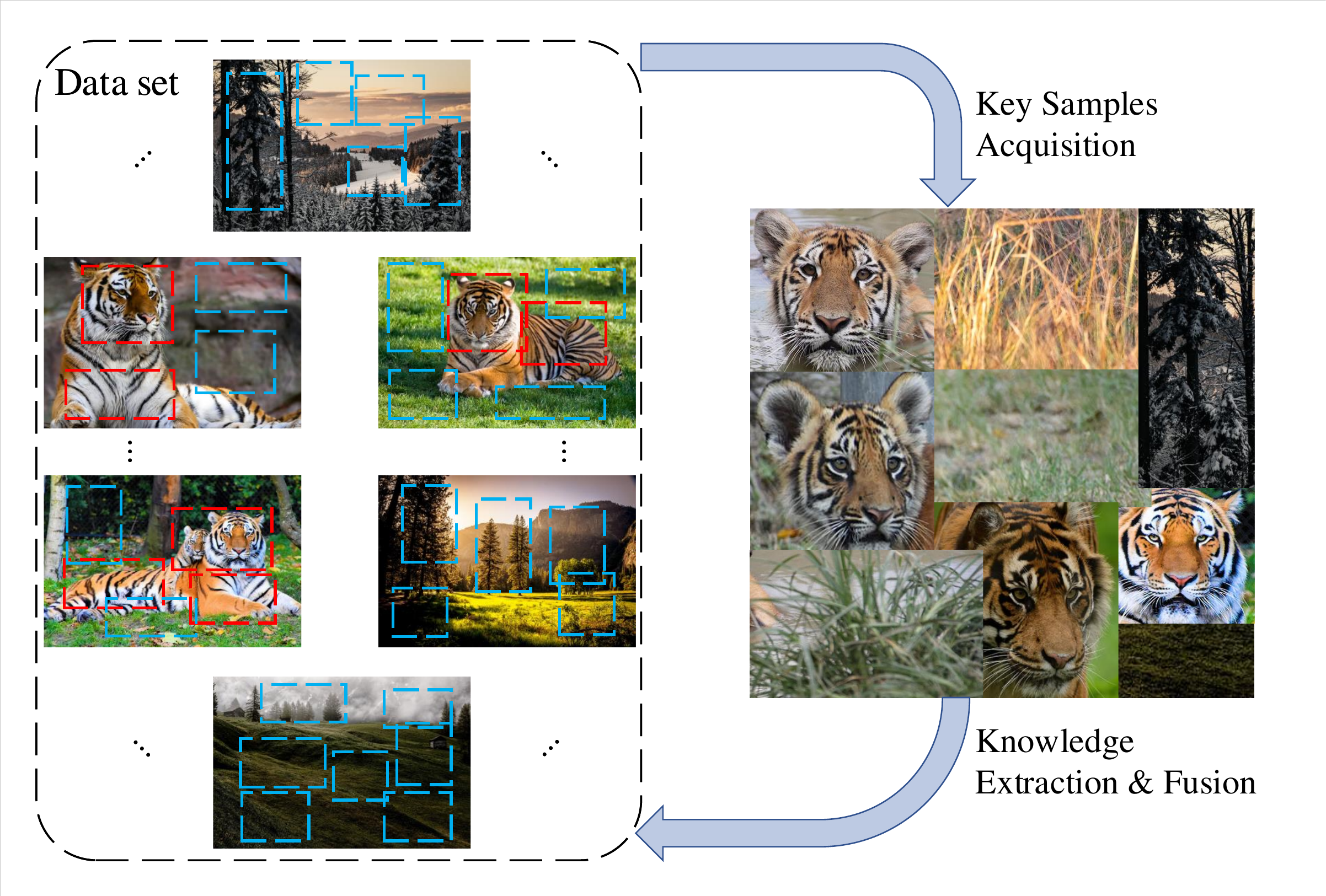}\label{fig: example}}
    \hspace{0.02\hsize}
    \subfloat[Model architecture]{\includegraphics[width=0.47\hsize]{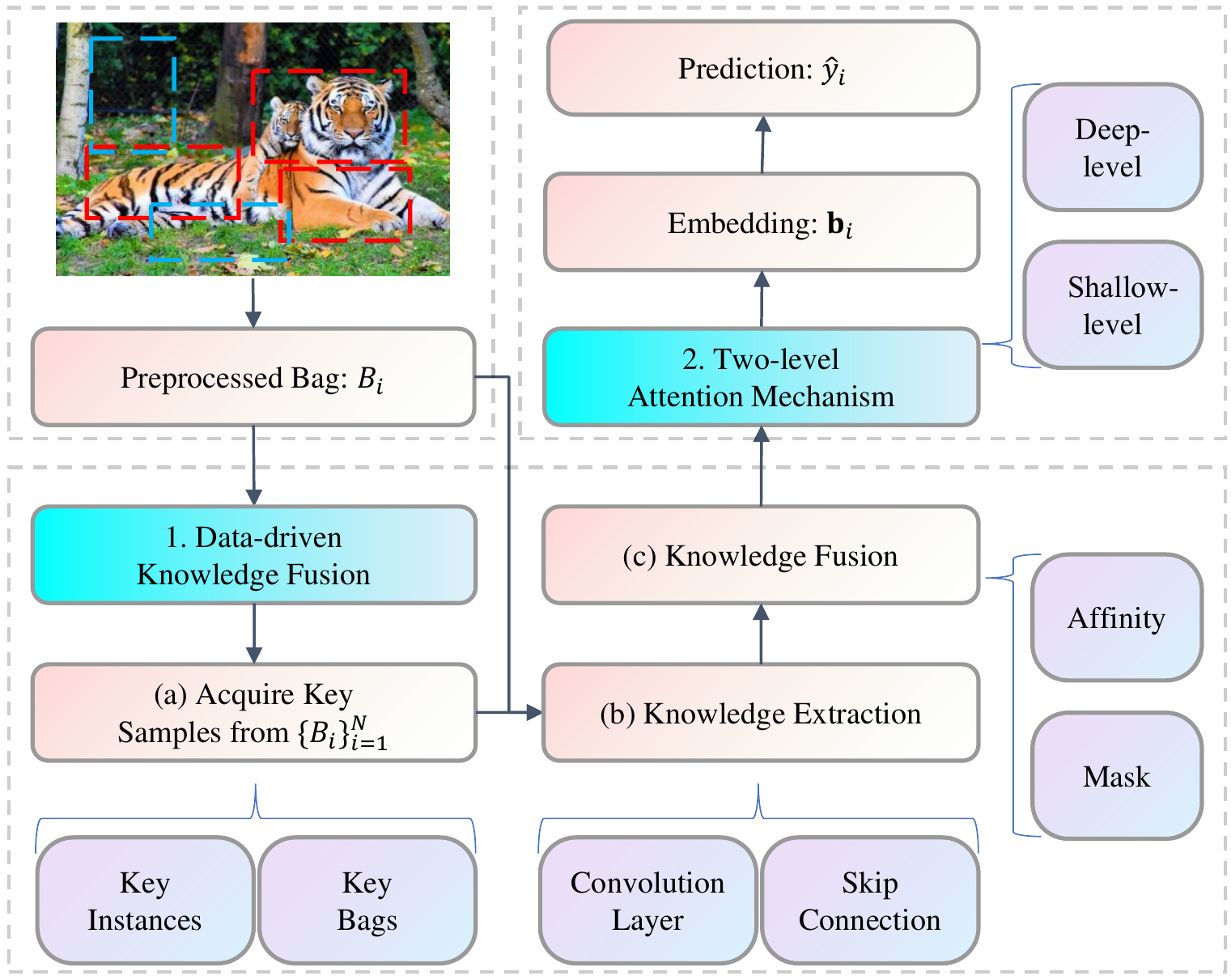}\label{fig: archi}}
    \caption{
    The semantic example and model architecture of DKMIL.
    All the images used in this example are from the open-source image library Pexel.
    For the sake of demonstration and description, we consider each image as a bag, with instances inside the bag being represented as patches.
    In practice, a bag may be a more complex scenario, such as the collection of various images, which will be further explained in the methodology and experiments.
    }
    \label{fig: example_archi}
\end{figure*}

These methods can be broadly categorized into traditional and deep learning-based methods due to their modeling strategies.
For traditional ones, the core is to utilize well-established machine learning classifiers such as SVM and $k$NN to assist in prediction, or use some mechanism to map bags into a new feature space to enable these classifiers to work \cite{Wei:2017:975987,Yang:2021:54565467}.
On the other hand, MIL deep learning approaches are essentially a simulation of the traditional methods, aiming to predict bag labels using the potent feature extraction capabilities of deep learning \cite{Ilse:2018:21272136,Konstantinov:2022:123}.
These methods are excellent and represent the exploration process and achievements of MIL.
However, a considerable amount of valuable knowledge is trapped inside the algorithm, leading to subsequent MIL algorithms that rely solely on the model's data fitting to predict unlabeled samples.
This represents a significant waste of knowledge and a hindrance to developing more intelligent models.
Therefore, we adopt a completely different idea from conventional MIL neural networks in that we introduce a distinct paradigm that pivots on a data-driven perspective, focusing on uncovering and exploiting the pre-existing knowledge stored within the algorithms.
This rich resource can be harnessed to extract valuable knowledge that can boost the model's learning ability.

Specifically, in this paper, we propose a data-driven knowledge fusion for deep multi-instance learning (DKMIL) algorithm.
Firstly, as depicted in Fig. \ref{fig: example}, for the provided data set, which includes multiple bags and instances, we can identify certain key samples, such as key instances of trigger bag labels and key bags at specific spatial locations, based on existing research such as \cite{Zhang:2009:4768,Wei:2017:975987,Wu:2018:10651080,Zhang:2020:16821689,Yang:2021:54565467}.
One of the notable contributions of these methods is to demonstrate the significance of these samples in training the model, showing that some of them may have a direct impact on the ultimate classification results.
Therefore, we devise a data-driven knowledge fusion module that serves as an interface between data and samples, enabling the extraction of valuable knowledge from key samples and leveraging the affinity matrix and mask to eliminate redundant information to learn the fusion bag, which is subsequently utilized for model training.
The advantage of this is that we take advantage of the established fact that key samples play a vital role in classification to aid in model training, and the flexibility of the model is greatly extended by the knowledge fusion module. The resulting scalability of the model allows the integration of any prior knowledge of existing methods without introducing excessive redundant information.
Additionally, in order to better utilize the knowledge-enriched features extracted from the data-driven knowledge fusion module and adapt them to the downstream modules of the model, we propose a two-level attention module that gradually learns both shallow- and deep-level features of the samples, resulting in a more effective classification process.
The complete architecture of DKMIL is shown in Fig. \ref{fig: archi}.

The main contributions of our work are as follows:
\begin{enumerate}
  \item
  We designed an algorithm that distinguishes itself from most previous MIL deep learning approaches by allowing for the fusion of prior knowledge from any method to aid in the model's learning process.
  The strong scalability of the model generated by this process has been demonstrated.
  \item
  We devised a data-driven knowledge fusion module that functions as an additional interface between the data and the model, enabling the model to learn fusion bags by using the knowledge of key samples.
  To make the downstream module of the model adaptable to fusion bags containing richer information, we design a two-level attention module that gradually learns shallow- and deep-level features of the samples, ultimately leading to more effective classification.
  \item
  We validated our DKMIL algorithm through extensive experiments on 38 data sets, covering three big scenarios.
  The results show that DKMIL exhibits the best overall classification performance while having the smallest parameter scale, which demonstrates the feasibility of improving the model by extracting prior knowledge from existing methods.
\end{enumerate}

\section{Related Works}

MIL was originally developed for predicting drug activity \cite{Dietterich:1997:3171}, which involves determining whether a bag of molecules (instances) contains those that can be used to make drugs.
The difficulty of this problem is that bag-level labels are provided, whereas instance-level labels are either unavailable or nonexistent.
This makes MIL a typical weakly supervised problem since the scarcity of labeled bags compared to the abundance of unlabeled instances.

Obviously, not all of these data hold equal importance for the classification process, just as not all molecules can be used to make drugs.
From a MIL research standpoint, certain key samples can be seen as intrinsic features of the data set, such as bags occupying vital positions in the data space and instances that play a crucial role in determining the label of the bag.
For example, BAMIC \cite{Zhang:2009:4768} and miVLAD \cite{Wei:2017:975987} approach the MIL problem from different angles by selecting key samples from the cluster centers of the bag space and instance space, respectively, using these samples to simplify the problem through the creation of mapping functions.
MILIS \cite{Fu:2011:958977} offers a distinctive solution that leverages an iterative optimization technique to pinpoint and remove instances that either do not affect the bag label or have a minimal impact on it.
Conversely, \cite{Tian:2021:49754986} uses a top-$k$ strategy to pick a certain number of instances that possess a higher likelihood of affecting the bag label.
MILDM \cite{Wu:2018:10651080} and ELDB \cite{Yang:2021:54565467} take into account the distribution of label space when selecting key samples, aimed at improving the interpretability and accuracy of the model's classification results.

Simultaneously, \cite{Ramon:2000:5360,Zhou:2002:455459,Zhang:2004:110} performed initial exploration to showcase the effectiveness of applying neural networks to MIL.
"Since then, numerous outstanding deep MIL algorithms have been proposed.
MI-Net \cite{Wang:2018:1524} employs deep supervision and residual connections to establish two efficient and scalable network structures.
ABMIL \cite{Ilse:2018:21272136} proposes two attention strategies for solving the MIL problem by representing the bag label as a fully neural network parameterized Bernoulli distribution.
LAMIL \cite{Shi:2020:57425745} introduces a loss function based on the attention mechanism that leverages the consistency cost to improve the generalization capability of the model.
DSMIL \cite{Li:2021:1431814328} presents a novel MIL aggregator to capture the instance relationships within a bag and uses self-supervised contrastive learning to obtain strong representations of the bags.
MAMIL \cite{Konstantinov:2022:123} considers the neighboring instances of each instance in a bag to handle various types of instances and produce a diverse feature representation for the bag.
HMIL \cite{Gao:2023:113} achieves better results with smaller data sets and improves generalization by exploiting the correlation among instances over different hierarchies.

The above methods offer many valuable lessons, such as the acquisition of key samples and the application of attention mechanisms.
Naturally, we aim to go beyond simply using using them as small modules in our methods or as comparison algorithms in experiments.
In reality, our goal is to extract the valuable knowledge contained in \cite{Zhang:2009:4768,Rodriguez:2014:14921496,Wei:2017:975987,Yang:2022:339351}, with the intention of aiding the training of the model.
As a result, the data-driven knowledge fusion and the two-level attention mechanism are designed to mimic this process.
More details about these techniques will be provided in the next section.

\section{Methodology}

This section outlines the problem setting and presents two main components of our DKMIL approach, namely the data-driven knowledge fusion and two-level attention mechanism.
We will then extend the model to make it more suitable for practical applications.
Finally, we prove the scalability of the data-driven knowledge fusion module based on the extended model.

\subsection{Problem Setting}\label{sec: problem_setting}

The primary target of this paper is to extract and fuse the bag-level and instance-level knowledge into our algorithm through a synthesis of the judgment made by existing MIL methods regarding key samples.
This is inspired by using knowledge graph reasoning from available data \cite{Hogan:2021:137} and abductive learning with logical grounding to assist inference \cite{Zhou:2019:13}.
The more important factor is that these successful MIL methods \cite{Zhang:2009:4768,Wu:2018:10651080} learned the relationship between the bags and key samples in order to achieve high classification performance in experiments across a variety of data domains.
Therefore, it is necessary to take into consideration the existence and potential influence of key samples, and then extract knowledge from these samples to support algorithm learning.

Let $\mathcal{D} = \{ B_i \}_{i=1}^N$ be the given data set, where $B_i = \{ \mathbf{x}_{ij} \}_{j=1}^{n_i}$ is a bag with the label $y_i$, $\mathbf{x}_{ij} \in \mathbb{R}^d$ is the $j$-th instance of $B_i$, $d$ is the dimension, $N$ and $n_i \in \mathbb{N}_1$ are the size of the data set and the bag, respectively.
By collecting all instances from $B_i$, the instance space is represented as $\mathcal{X} = \bigcup_i B_i$.
Note that the instances in $\mathcal{X}$ are renumbered to $\mathbf{x}_{l}$, where $l \in (1,L]$ and $L = \sum_i n_i$.
The learning process of DKMIL can be seperated into the following two parts, as shown in Fig. \ref{fig: archi}:

First, the key instance set $\mathcal{I} = \{ \mathbf{x}_k^{ins} \}_{k = 1}^{N_1}$ and key bag set $\mathcal{B} = \{ B_k^{bag} \}_{k = 1}^{N_2}$ are respectively generated from $\mathcal{X}$ and $\mathcal{D}$, and the fusion bag $B_i^{fuse}$ is obtained with the extraction and fusion blocks:
\begin{equation}\label{eq: step1.1}
B_i^{*}, \mathcal{I}^{*}, \mathcal{B}^{*} = KnowledgeExtraction(B_i, \mathcal{I}, \mathcal{B}),
\end{equation}
\begin{equation}\label{eq: step1.2}
B_i^{fuse} = KnowledgeFusion(B_i^{*}, \mathcal{I}^{*}, \mathcal{B}^{*}),
\end{equation}
where $\mathbf{x}_k^{ins} \in \mathcal{X}, B_k^{bag} \in \mathcal{D}$, $N_1 = | \mathcal{I} |$, and $N_2 = | \mathcal{B} |$.

Second, the fused representation will be used to obtain the prediction $\hat{y}_i$ via the two-level attention mechanism:
\begin{equation}\label{eq: step2}
\hat{y}_i = Classifier(Attention(B_i^{fuse})).
\end{equation}
In the following chapters, we will delve deeper into the implementation of Eqs. \eqref{eq: step1.1}--\eqref{eq: step2}.

\subsection{Data-driven Knowledge Fusion}

Data-driven knowledge fusion first performs an analysis of the data set using some data evaluation criteria from existing MIL methods, and identifies key instances and bags.
Based on these, the knowledge extraction block uses the convolution with skip connection to extract valuable information.
The knowledge fusion block then integrates these representations using a mask affinity matrix, generating a comprehensive bag representation that will serve as input for attention blocks.

\subsubsection{Key Sample Acquisition}\label{sec: al_key_sample}

The significance of the key sample acquisition module, as the first sub-module of the data-driven knowledge fusion module, cannot be overstated.
Certainly, it is a given that methods for obtaining key samples already exist.
However, the main objective of this subsection is not to simply list these methods but to summarize a series of methods that can provide important metrics for the subsequent module design.
Specifically, as the fundamental component of their algorithm architecture, key sample-based MIL approaches consider the existence of some key instances or bags in the data set, which can have a direct impact on the classification result.
For example, BAMIC \cite{Zhang:2009:4768} and miVLAD \cite{Wei:2017:975987} believe that the cluster centers can serve as key samples.
MILDM \cite{Wu:2018:10651080} and ELDB \cite{Yang:2021:54565467} hold that the distinguishability between intermediate representations should be taken into account when choosing key ones.
Based on these, we summarize the following two evaluators to choose key instances and key bags, respectively.
For instance $\mathbf{x}_l \in X \subseteq \mathcal{X}$, its centrality \cite{Zhang:2009:4768,Wei:2017:975987,Yang:2022:339351} is defined as:
\begin{equation}\label{eq: centrality}
s_{l}^{cen} = \frac{| X |}{\sum_{t = 1, \xi_t \neq l}^{|X|} \| \mathbf{x}_l - \mathbf{x}_{\xi_t} \|_2},
\end{equation}
where $\xi_t \in [1,L]$.
The density \cite{Rodriguez:2014:14921496} of $\mathbf{x}_l$ is defined as:
\begin{equation}\label{eq: density}
s_l^{den} = \rho_l \times \delta_l,
\end{equation}
where
$$
\rho_l = \sum_{t = 1, \xi_t \neq l}^{| X |}e^{-\left( \frac{\| \mathbf{x}_l - \mathbf{x}_{\xi_t} \| }{0.5 \times \tau} \right)^2},
$$
and
$$
\delta_l=
\left\{
    \begin{array}{ll}
        \tau,                                                                 & \rho_l = \max_{t} \rho_{\xi_t};\\
        \min_{\rho_{\xi_t} > \rho_l} \| \mathbf{x}_l - \mathbf{x}_{\xi_t} \|, & \operatorname{otherwise},\\
    \end{array}
\right.
$$
where $\tau = \max_{t} \| \mathbf{x}_l - \mathbf{x}_{\xi_t} \|$.
The two formulas' central ideas diverge.
According to the Eq. \eqref{eq: centrality}, instance $\mathbf{x}_l$ is more likely to be identified as a key instance if the reciprocal of the average distance between it and other instances is larger.
Equation \eqref{eq: density} contends that there should be a small similarity between key instances and, as a result, multiplies an adaptive distance $\delta_l$ to rectify the RBF kernel's output $\rho_l$.
Algorithm \ref{ag: key_instance_random} illustrates a random sampling technique that utilizes these two evaluators to obtain key instances $\mathcal{I}$ from the data set $\mathcal{D}$.
A visual depiction of this algorithm on the $\mathcal{D}$ can be seen in Fig. \ref{fig: ins_sampling}.
The results indicate that the random sampling technique does not miss the most key instances evaluated with Eqs. \eqref{eq: centrality} and \eqref{eq: density} (a mathematically proof is then given).
Additionally, some other instances might be included to enhance the diversity of sampling results.

\begin{algorithm}[!t]
\caption{Key instance acquisition}\label{ag: key_instance_random}
 \KwData{$\ $\\
     $\qquad$Data set $\mathcal{D}$;\\
     $\qquad$Number of random sampling $m_n$;\\
     $\qquad$Sampling space $X$'s size $m_s \leq L$;\\
     $\qquad$Number of key instances acquired from the\\
     $\qquad$current sampling space $m_k < m_s$;
 }
 \KwResult{$\ $\\
    $\qquad$Key instance set $\mathcal{I}$;
 }
 Initialize $\mathcal{I} = \emptyset$ and generate $\mathcal{X}$ based on $\mathcal{D}$\;
 \For{$i \in [1, m_n]$}
 {
    Generate the sampling space $X \subseteq \mathcal{X}$, where $| X | = m_s$\;
    Compute the centrality $s_{\xi_t}^{cen}$ for each instance $\mathbf{x}_{\xi_t} \in X$ via Eq. \eqref{eq: centrality}\;
    Compute the density $s_{\xi_t}^{den}$ for $\mathbf{x}_{\xi_t}$ via Eq. \eqref{eq: density}\;
    Get $m_k$ instance with maximum centrality and add them to $\mathcal{I}$\;
    Get $m_k$ instance with maximum density and add them to $\mathcal{I}$\;
 }
 \Return $\mathcal{I}$\;
\end{algorithm}

\begin{theorem}\label{the: random}
    Let $\mathcal{X}$ be the instance space, $m_n$ be the number of random sampling, $m_s$ be the sample space's size, $m_k$ be the number of key instances acquired, and $L$ is be size of $\mathcal{X}$.
    Assume $m_n$ is large enough, we can conclude that the random sampling technique is capable of selecting the most key instances.
\end{theorem}

\begin{proof}\label{proof: random}
    Take the density calculation in Eq. \eqref{eq: density} as an example and assume that the densities of all instances are the same.
    For any given instance from $\mathcal{X}$, the probability that it is selected in each sample is $m_k / L$.
    Additionally, in multiple sampling, the probability that each instance is sampled at least once is $= \lim_{m_n \to \infty} 1 - (1 - m_k / L)^{m_n} = 1$.
    Then the average number of samples is $m_k m_n / L$.
    Instance densities, however, are often inconsistent, which means that the probability of a high-density instance being selected is much greater than that of a low-density instance.
    As a result, the upper and lower bounds of the average sampling times are $m_s m_n / L$ and $0$, respectively.
    Based on this, we can conclude that the random sampling technique is capable of selecting the most key instances.
\end{proof}

\begin{figure}[!htb]
\centering
\includegraphics[width=0.85\linewidth]{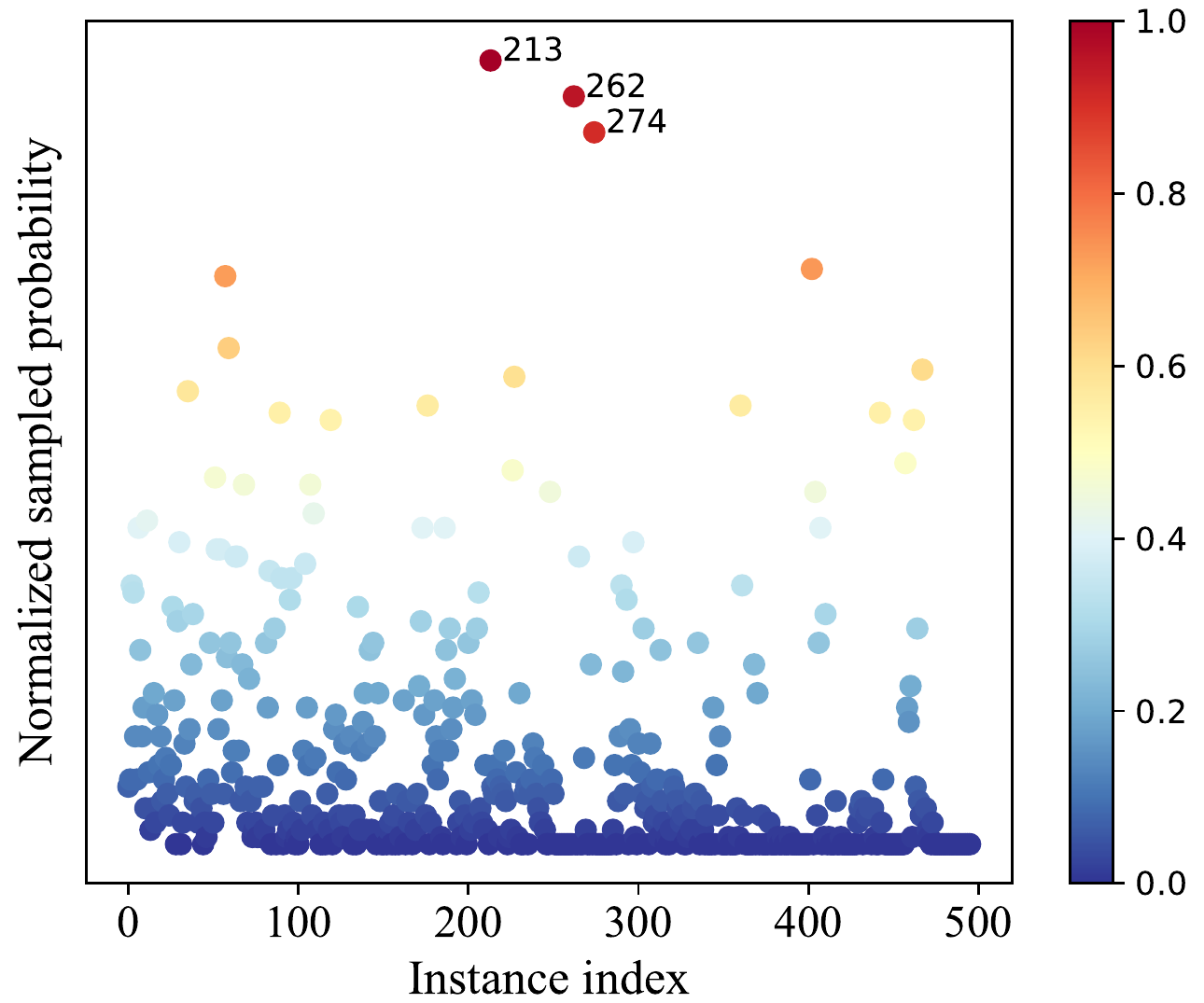}
\caption{
Schematic view of key instance acquisition using random sampling on the musk1 data set \cite{Dietterich:1997:3171} under $m_n = 1000$.
The abscissa denotes the index of all instances in the data set, and the ordinate indicates the sampling probability, which is normalized by dividing by the maximum of all probability values.
}
\label{fig: ins_sampling}
\end{figure}

For bag-level analysis, the calculation of centrality $S_i^{cen}$ and density $S_i^{den}$ for bag $B_i$ is almost the same as that for instance-level ones.
The only difference is the replacement of the $l_2$-norm between instances (i.e. $\| \cdot \|_2$) with a bag-to-bag distance metric since the sampling space consists bags.
Note that $S_i$ is used to distinguish between bag- and instance-level calculations.
Recent studies have summarized over ten distance metrics, among which the MSK metric with linear time complexity is proposed \cite{Yang:2022:339351}.
MSK has the highest overall classification performance, although not particularly outstanding performance on domain-specific data sets.
Therefore, it is also utilized as the distance metric between bags in our approach, which may provide better scalability and adaptability.

\subsubsection{Knowledge Extraction Block}

The primary objective of the knowledge extraction module is to extract valuable information from the input bag $B_i$, key instance set $\mathcal{I}$, and key bag set $\mathcal{B}$ for subsequent learning.
We have opted to utilize convolutional neural networks and skip connections as part of this module for the following rationales.
Convolution neural networks \cite{Albawi:2017:16,Ilg:2017:24622470} are widely recognized for their ability to learn input-output relationships based on labeled data.
And the skip connection \cite{Tong:2017:47994807,Wang:2022:24412449} allows for the construction of short paths from the output to the input, which mitigates the vanishing-gradient problem encountered in deep networks and retains the characteristics of the input data.
Because of this, we build the knowledge extraction block based on these two useful components, and its entire design is illustrated in Fig. \ref{fig: skip-connect}.

For the given input $B_i$, three convolution layers are firstly used to extract the latent knowledge of the input $B_i$, i.e., $B_i^{conv1} = Conv1 (B_i^T)$, $B_i^{conv2} = Conv2 (B_i^T)$, and $B_i^{conv3} = Conv3 (B_i^T)$, where kernel size and stride are both set to $1$.
Because there are varying numbers of instances in the various bags, the input $B_i$ needs to be transposed in this case.
So that the convolutional layer's input channel number can be easily set to $d$ for all bags.
The relationship between instance pairs is then evaluated by fusing $B_i^{conv1}$ and $B_i^{conv}$ to get $B_i^{rela} = Softmax((B_i^{conv1})^T \times B_i^{conv2})$, where $\times$ stands for matrix multiplication and $B_i^{rela} \in \mathbb{R}^{n_i \times n_i}$.
The reason for this is that the bags in the data set often contain redundant instances, and special techniques are necessary to filter them out.
One such technique is the subspace fuzzy clustering method used in FCBE-miFV \cite{Waqas:2022:119113}.
Therefore, we use $B_i^{rela} \times (B_i^{conv3})^T$ to simulate this process, as if the relationship value $b_{jk}^{rela} \in B_i^{rela}$ between two instances $\mathbf{x}_{ij}$ and $\mathbf{x}_{ik}$ is low, the values of the corresponding learned features will tend to be zero.
Finally, the knowledge derived from $B_i$ is computed as $B_i^* = B_i \oplus B_i^{rela} \times (B_i^{conv3})^T$ based on the main concept of skip connection.
Similarly, $\mathcal{I}^*$ and $\mathcal{B}^*$ can be computed as:
\begin{equation}
\mathcal{I}^* = SkipConnection(\mathcal{I}),
\end{equation}
\begin{equation}
\mathcal{B}^* = \{ SkipConnection(B_k^{bag}) \}_{k = 1}^{N_2}.
\end{equation}

\begin{figure}[!htb]
\centering
\includegraphics[width=0.75\linewidth]{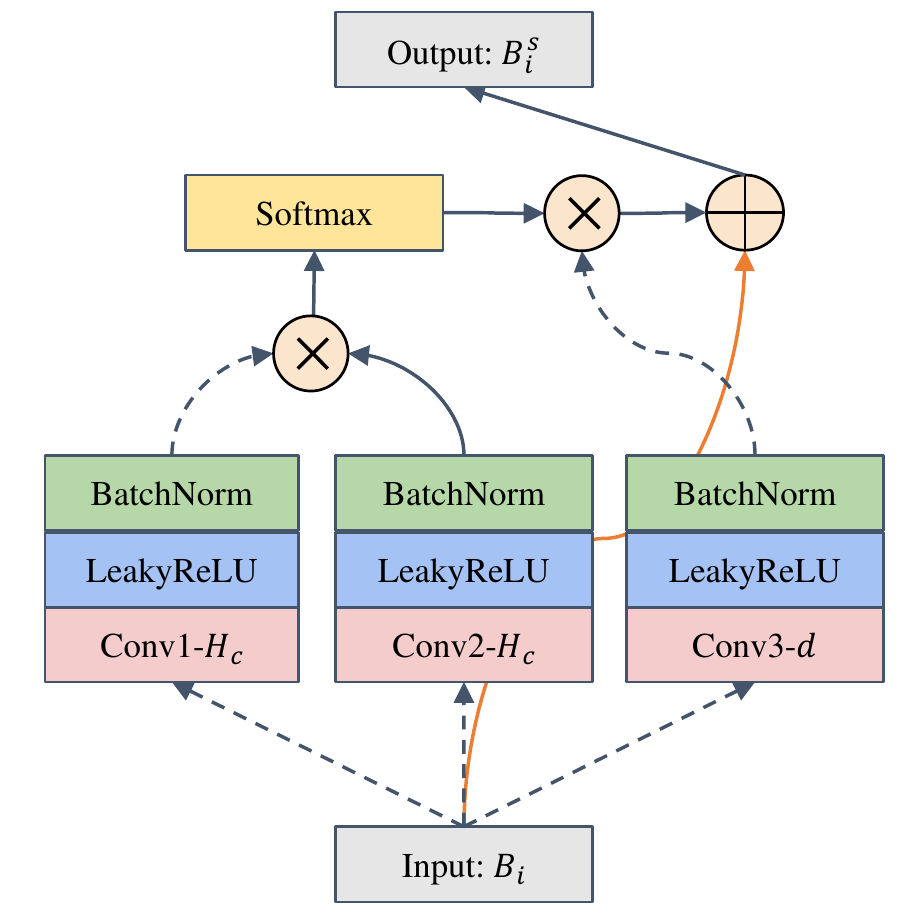}
\caption{
The architecture of knowledge extraction block.
$H_c$ and $d$ represent both the number of output channels in the convolutional layer, and $d$ also denotes the dimension of instances in $B_i$.
Dashed and solid lines with arrows signify the need for transposed input and direct input, respectively.
The orange solid line with an arrow represents a short path in the skip connection mechanism.
The circles with ``$\times$'' and ``$\oplus$'' denote matrix multiplication and element-wise addition, respectively.
}
\label{fig: skip-connect}
\end{figure}

\subsubsection{Knowledge Fusion Block}

The knowledge fusion module constitutes a critical component of our algorithm, as it serves as the direct interface between the key samples and the model.
It not only facilitates the integration of key sample knowledge into the bag but also eliminates the potentially disruptive information that may hinder the model's fitness.
The presence of redundant information is attributed to the random sampling method used in key sample acquisition, which results in the inclusion of non-key samples.
Additionally, there may be mutual exclusivity between different key sample evaluation metrics, as the determination of whether a sample is a key sample or not may yield opposite results under different metrics.
Our proposed solution is to model the correlation between the bag and the key instance set as well as the key bag set by using affinity matrixes.
We then employ a mask to eliminate certain features and obtain the fusion bag.

Specifically, we have now obtained three abstract knowledge $B_i^*, \mathcal{I}^*$, and $\mathcal{B}^*$, which contain valuable information, such as instances that trigger bag labels \cite{Yang:2022:339351} and bags that can build embedding functions \cite{Zhang:2009:4768}.
One crucial tool for utilizing this knowledge is the affinity matrix \cite{Zhou:2009:12491256}, where each element indicates the degree of association between two key samples.
Specifically, the affinity matrix $\mathcal{A}_i^\mathcal{I}$ between a given bag $B_i = \{\mathbf{x}_{ij}\}_{j = 1}^{n_i}$ and $\mathcal{I}^*$ is defined as:
\begin{equation}
\mathcal{A}_i^\mathcal{I} =
\left[
\begin{array}{ccc}
\| \mathbf{x}_{i1}^* - \mathbf{x}_1^{*} \|_2 & \cdots & \| \mathbf{x}_{i1}^* - \mathbf{x}_{N_1}^{*} \|_2 \\
\vdots & \ddots & \vdots\\
\| \mathbf{x}_{in_i}^* - \mathbf{x}_1^{*} \|_2 & \cdots & \| \mathbf{x}_{in_i}^* - \mathbf{x}_{N_1}^{*} \|_2 \\
\end{array}
\right],
\end{equation}
where $\mathbf{x}_{ij}^* \in B^*$ and $\mathbf{x}_k^* \in \mathcal{I}^*$ are the deep-level representations of $\mathbf{x}_{ij} \in B_i$ and $\mathbf{x}_k^{ins} \in \mathcal{I}$, respectively.
Similarly, the affinity matrix $\mathcal{A}_i^\mathcal{B}$ between $B_i$ and $\mathcal{B}^*$ is calculated as:
\begin{equation}
\mathcal{A}_i^\mathcal{B} =
\left[
\begin{array}{ccc}
a_i^{11} & \cdots & a_i^{1N_2} \\
\vdots & \ddots & \vdots\\
a_i^{n_i1} & \cdots & a_i^{n_iN_2} \\
\end{array}
\right],
\end{equation}
where
$$
a_i^{jk} =
\left[
\| \mathbf{x}_{ij}^* - \mathbf{x}_{k1}^* \|_2, \dots, \| \mathbf{x}_{ij}^* - \mathbf{x}_{kn_k^*}^* \|_2
\right],
$$
where $\mathbf{x}_{k\cdot} \in SkipConnection(B_k^{bag})$ and $n_k^*$ is cardinality of $B_k^{bag}$.

The issue currently is that $\mathcal{A}_i^\mathcal{I}$ and $\mathcal{A}_i^\mathcal{B}$ still contain significant amounts of redundant information, and an effective strategy is required to extract knowledge from them and fuse the results with $B_i^*$.
Mask is a common operation in deep learning \cite{He:2017:29612969,Wei:2018:704714}, which involve adding a mask to the original input to block or select some specific elements.
In this work, we introduce a mask block that utilizes the fundamental principle of mask to generate two mask affinity matrices.
This block is constructed as follows:
\begin{align}
l_{i1}^{mask} &= LeakyReLU(\mathcal{A}_i \times W_1^{mask}),\\
l_{i2}^{mask} &= Tanh(\mathcal{A}_i \times W_2^{mask}),\\
l_{i3}^{mask} &= LeakyReLU((l_{i1}^{mask})^T \times l_{i2}^{mask} \times  W_3^{mask}),\\
l_{i4}^{mask} &= Softmax((l_{i3}^{mask})^T \times W_4^{mask}),
\end{align}
where $W_1^{mask},W_2^{mask} \in \mathbb{R}^{d \times H_m}$, $W_3 \in \mathbb{R}^{H_m \times d}$, and $W_4 \in \mathbb{R}^{H_m \times 1}$ are the weight parameters for this block.
$H_m$ is the number of nodes.
For convenience of description, we omit all bias parameters.
In particular, $l_{i4}^{mask}$ of size $1 \times d$ takes each element as the importance of the corresponding feature in the input data, with the sum of all elements being equal to $1$.
Algorithm \ref{ag: mask} outlines how to use the mask block to generate the mask affinity matrix $\mathcal{M}_i^\mathcal{I}$ and $\mathcal{M}_i^\mathcal{B}$.

Following are some more explanations to help understand this algorithm:
a) Skip connection is used to extract more useful features from the input affinity matrices;
b) $DescendingArgSort(\cdot)$ stands for the index in descending order of importance, which is used to filter out irrelevant information; and
c) The dimension of the mask affinity matrix based on $\mathcal{A}_i^\mathcal{B}$ is only related to $N_2$ due to considering the positive bag has at least one positive instance \cite{Dietterich:1997:3171} and the negative bag can choose an instance as its own representation \cite{Yang:2022:122133}.

Once we have $\mathcal{M}_i^\mathcal{I}$ and $\mathcal{M}_i^\mathcal{B}$, we can stack them with $B_i$ in the feature dimension:
\begin{equation}
B_i^{stack} = Stack(B_i, \mathcal{M}_i^\mathcal{I}, \mathcal{M}_i^\mathcal{B}),
\end{equation}
where $B_i^{stack} \in \mathbb{R}^{n_i \times (d + d^{mask} + N_2) }$.
Finally, the bag that fuses abstract knowledge extracted from key instances and key bags is expressed as:
\begin{equation}
B_i^{fuse} = LeakyReLU (B_i^{stack} \times W^{fuse}),
\end{equation}
where $W^{fuse} \in \mathbb{R}^{(d + d^{mask} + N_2) \times d}$.

\begin{algorithm}[!htb]
\caption{Generate the mask affinity matrices}\label{ag: mask}
 \KwData{$\ $\\
     $\qquad$Affinity matrices $\mathcal{A}_i^\mathcal{I}$ and $\mathcal{A}_i^\mathcal{B}$;\\
     $\qquad$Minimum dimensions $N_m$;\\
     $\qquad$Mask's ratio parameter $r$.
 }
 \KwResult{$\ $\\
    $\qquad$Mask affinity matrices $\mathcal{M}_i^\mathcal{I}$ and $\mathcal{M}_i^\mathcal{B}$;
 }
 Compute $l_{i4}^{mask}$ using $SkipConnection(\mathcal{A}_i^\mathcal{I})$\;
 $\varsigma_i$ = $DescendingArgsort(l_4^{mask})$\;
 $d^{mask} = \max (N_m, r \times N_1)$\;
 Update $\varsigma_i$ by selecting the first $d^{mask}$ indices in $\varsigma_i$\;
 Generate $\mathcal{M}_i^\mathcal{I}$ by selecting the columns corresponding to $\varsigma_i$ from $SkipConnection(\mathcal{A}_i^\mathcal{I})$\;
 Reset $\varsigma_i = \emptyset$\;
 \For{$k \in [1, N_2]$}
 {
    Compute $l_4^{mask}$ using the $k$-th column of $SkipConnection(\mathcal{A}_i^\mathcal{B})$\;
    $\varsigma_{ik} = \argmax l_{i4}^{mask}$\;
    $\varsigma_i \leftarrow \varsigma_i \cup \{ \varsigma_{ik} \}$\;
 }
 Generate $\mathcal{M}_i^\mathcal{B}$ by selecting the columns corresponding to $\varsigma_i$ from $SkipConnection(\mathcal{A}_i^\mathcal{B})$\;
 \Return $\mathcal{M}_i^\mathcal{I}$ and $\mathcal{M}_i^\mathcal{B}$\;
\end{algorithm}

\subsection{Two-level Attention Mechanism}

We developed a data-driven knowledge fusion technique to combine the abstract knowledge of the key instance set and the key bag set.
However, the challenge remains in constructing a deep learning classifier capable of handling the variable size $n_i$ of $B_i^{fuse}$ and obtaining the bag prediction label $\hat{y}_i$.
The attention mechanism \cite{Vaswani:2017:111}, mimicking cognitive attention in artificial neural networks, is often used to enhance some parts of the input data and reduce others.
In MIL, it helps determine the weight of instances and direct the learner's focus \cite{Ilse:2018:21272136}.
Based on this, we will construct a two-level attention block to improve classification.

\subsubsection{Shallow-level Attention Block}

The shallow-level attention block is utilized to extract features from $B_i^{fuse}$ and fuse them into an embedding vector, serving as an initial step for classification.
Its structure is as follows:
\begin{align}
l_{i1}^{low}    & = LeakyReLU(B_i^{fuse} \times W_1^{low}),\\
l_{i2}^{low}    & = TanH(l_{i1}^{low} W_2^{low}),\\
l_{i3}^{low}    & = LeakyReLU(l_{i1}^{low} \times W_3^{low}),\\
l_{i4}^{low}    & = LeakyReLU((l_{i2}^{low} \otimes l_{i3}^{low}) \times W_4^{low}),\\
\alpha_i^{low} & = Softmax((l_{i4}^{low})^T),\\
\mathbf{b}_i^{low}    & = LeakyReLU(\alpha^{low} \times l_{i1}^{low} \times W_5^{low}),
\end{align}
where $W_1^{low} \in \mathbb{R}^{d \times H_l}$, $W_2^{low}, W_3^{low} \in \mathbb{R}^{H_l \times D_l}$, $W_4^{low} \in \mathbb{R}^{D_l \times 1}$, $W_5^{low} \in \mathbb{R}^{H_l \times d}$, $\otimes$ represents the element-wise multiplication, and $H_l,D_l$ are the number of nodes.
$\alpha^{low} \in \mathbb{R}^{1 \times n_i}$ is called the attention value (a.k.a. instance weight), which weights the features extracted from $B_i^{fuse} \in \mathbb{R}^{n_i \times d}$ into the shallow-level embedding $\mathbf{b}_i^{low} \in \mathbb{R}^{1 \times d}$ for upcoming decision-making.

\subsubsection{Deep-level Attention Block}

The core requirement of the deep-level attention block is to use $B_i^{fuse}$ and $\mathbf{b}_i^{low}$ to complete the final prediction, which consists of the shallow-level attention block followed by a classification layer:
\begin{equation}
\hat{y}_i = Sigmoid(\mathbf{b}_i^{high} \times W^{high}),
\end{equation}
where $W^{cla} \in \mathbb{R}^{H^h \times 1}$, $H^h$ is the number of nodes, and $\mathbf{b}_i^{high}$ is computed by feeding
\begin{equation}
B_i^{stack} = \{ Stack(\mathbf{b}_i^{low}, \mathbf{x}_{ij}^{fuse}) \},
\end{equation}
into the shallow-level attention block, where $\mathbf{x}_{ij}^{fuse}$ is the deep-level representation of $\mathbf{x}_{ij}$ in $B_i^{fuse}$.
The information contained in $\mathbf{b}_i^{low}$, which is derived from the abstraction of knowledge from $B_i$, $\mathcal{I}$, and $\mathcal{B}$, motivates us to stack it together with $B_i^{fuse}$ in the attention block. By doing so, we aim to enable the block to learn more informative attention values, which will be experimentally validated.

\subsection{Model Extension}

The aforementioned algorithm is designed and described for bags in which the instance is represented by a vector.
It is evident that it cannot handle more complex MIL applications, such as determining if an image bag retrieved from a search engine contains the images that we are looking for.
To demonstrate how to extend DKMIL to handle such cases, we provide the example of the image application.
In this case, a data-transformation block is required to map $B_i^{image} = \{ I_i \}_{i = 1}^{n_i}$ to the space of $B_i$, where $I_i$ represents an image:
\begin{align}
l_{i1}^{map} & = MaxPool2d(LeakyReLU(Conv2d\text{-}H_d(I_i))),\\
l_{i2}^{map} & = MaxPool2d(LeakyReLU(Conv2d\text{-}D_d(I_i))),\\
l_{i3}^{map} & = Reshape(l_{i2}^{map}),
\end{align}
where $MaxPool2d$ and $Conv2d$ are respectively the max pooling and convolution blocks used for the image, $H_d$ and $D_d$ are the number of output channels.
Here $Reshape$ means change the shape of $l_{i2}^{map}$ to $(n_i, h \times w)$, where $h$ and $w$ are the height and width of $l_{i2}^{map}$.

However, this presents a new challenge:
the criteria developed in Section \ref{sec: al_key_sample} for choosing key samples are no longer applicable.
First, it is possible to calculate the similarity between two images using the Euclidean distance, but important information such as the correlation between the upper and lower rows may be lost.
Second, there is no existing metric to calculate the similarity between two bags of images, so within the scope of our knowledge, searching for key bags is pointless.
Therefore, in this more complex scenario where images are treated as instances, we will only focus on finding key images and utilize MS-SSIM \cite{Wang:2003:13981402} as a measure of their similarity, i.e., the $\|\cdot\|_2$ in Eqs. \eqref{eq: centrality} and \eqref{eq: density} will be replaced by the MS-SSIM value of the two images.
Moreover, the dimensions of some data sets, such as the web data set \cite{Zhou:2005:135147}, are much larger than the number of bags, which may result in processing irrelevant information and prolong the runtime of the algorithm.
To tackle this issue, w adopt a simple solution by introducing a fully connected layer to reduce the dimensionality of $B_i$ before entering the network:
\begin{equation}
B_i \leftarrow LeakyReLU(B_i \times W^{dim}),
\end{equation}
where $W^{dim} \in \mathbb{R}^{d \times N}$.

\subsection{Discussion}

Our core innovation lies in the data-driven knowledge fusion module, which represents a novel approach to leveraging the rich prior knowledge in MIL.
By designing an effective interface between data and models, we are able to extract useful knowledge from key samples, thereby enhancing the machine learning capabilities.
Obviously, such a module is highly scalable.

\begin{theorem}\label{the: scalable}
    Let $\mathcal{K} (B_i^{fuse} | B_i, \mathcal{D})$ be the data-driven knowledge fusion module and $\mathcal{D}$ be a data set, where $B_i \in \mathcal{D}$ is a bag, we can conclude that $\mathcal{K} (\cdot)$ is a strongly scalable module.
\end{theorem}

\begin{proof}\label{proof: scalable}
In accordance with the current framework of $\mathcal{K} (\cdot)$, we can derive the key instance set $\mathcal{I} \subset \mathcal{X} = \bigcup B_i$ from $\mathcal{D}$ based on Algorithm \ref{ag: key_instance_random}.
Correspondingly, we can acquire the key bag set $\mathcal{B} \subset \mathcal{D}$ in a similar manner.
Then we have
\begin{equation}\label{eq: scalable1}
    \mathcal{K} (B_i^{fuse} | B_i, \mathcal{D}) = \mathcal{K} (B_i^{fuse} | B_i, \mathcal{I}, \mathcal{B}).
\end{equation}

Clearly, the number of possible cases of $\mathcal{I}$ is
$$
    \sum_{l \in [1, L)} C_L^l = 2^L - 2,
$$
where $L = | \mathcal{X} | = \sum_i n_i$ and $n_i $ is the size of $B_i$.
Accordingly, the number of possible cases of $\mathcal{B}$ is $2^N - 2$, where $N = | \mathcal{D} |$.
In other words, by employing various distinctive strategies for obtaining key samples, such as the centrality of Eq. \eqref{eq: centrality} or the density of Eq. \eqref{eq: density}, we can generate close to $2^L + 2^N$ number of feasible alternatives.
Therefore, Eq. \eqref{eq: scalable1} can be rewritten as
\begin{equation}\label{eq: scalable2}
    \mathcal{K} (B_i^{fuse} | B_i, \mathcal{D}) = \mathcal{K} (B_i^{fuse} | B_i, \mathcal{D}, S),
\end{equation}
where $S = \{ s_l^{cen}, s_l^{den}, S_i^{cen}, S_i^{den} \}$, and $s_l^{cen}$/$S_i^{cen}$ and $s_l^{den}$/$S_i^{den}$ are the centrality and density of $\mathbf{x}_l$/$B_i$, where $\mathbf{x}_l \in \mathcal{X}$.

As per Theorem \ref{the: random}, the probability of many instances or bags being selected is essentially zero, which implies that the number of viable options at our disposal will be significantly less than $2^L + 2^N$.
Nevertheless, there still exist numerous distinct methods $\mathcal{S}$ for selecting the key samples, which, in conjunction with Algorithm
\ref{ag: key_instance_random}, contribute to the scalability of $\mathcal{K} (\cdot)$.
This scalability is manifested through the diversity of samples $\mathcal{I}$ and $\mathcal{B}$ and methods $S \in \mathcal{S}$ to obtain key samples.

In addition, our current approach is data-driven, so the prior knowledge utilized to obtain the fusion bag $B_i^{fuse}$ comprises only $\mathcal{I}$ and $\mathcal{B}$.
However, a vast array of knowledge can be summarized from existing algorithms, such as abductive learning, which incorporates logical reasoning, and graph neural networks with the introduction of graph structures.
Both of these approaches can potentially yield knowledge for obtaining $B_i^{fuse}$ under different knowledge fusion modules.
Therefore, in the end we have
\begin{equation}\label{eq: scalable2}
    \mathcal{K} (B_i^{fuse} | B_i, \mathcal{D}) = \mathcal{K} (B_i^{fuse} | B_i, \mathcal{D}, \mathcal{S}),
\end{equation}
Consequently, we can say that $\mathcal{S}$ ensures the strong scalability of $\mathcal{K} (\cdot)$.
\end{proof}

\section{Experiments}

In this section, we will verify our algorithm through six experiments, including an ablation study and performance, convergence, statistical significance, and vulnerability comparisons.
They will be used to demonstrate the importance of data-driven knowledge fusion and two-level attention, as well as the effectiveness, security, and others of the algorithm.
Prior to that, we will detail the comparison algorithms and their parameters, the data sets that were utilized, and the evaluation metrics in the parameter setups.

\begin{figure*}[!htb]
    \centering
    \subfloat[Musk1]{\includegraphics[width=0.33\hsize]{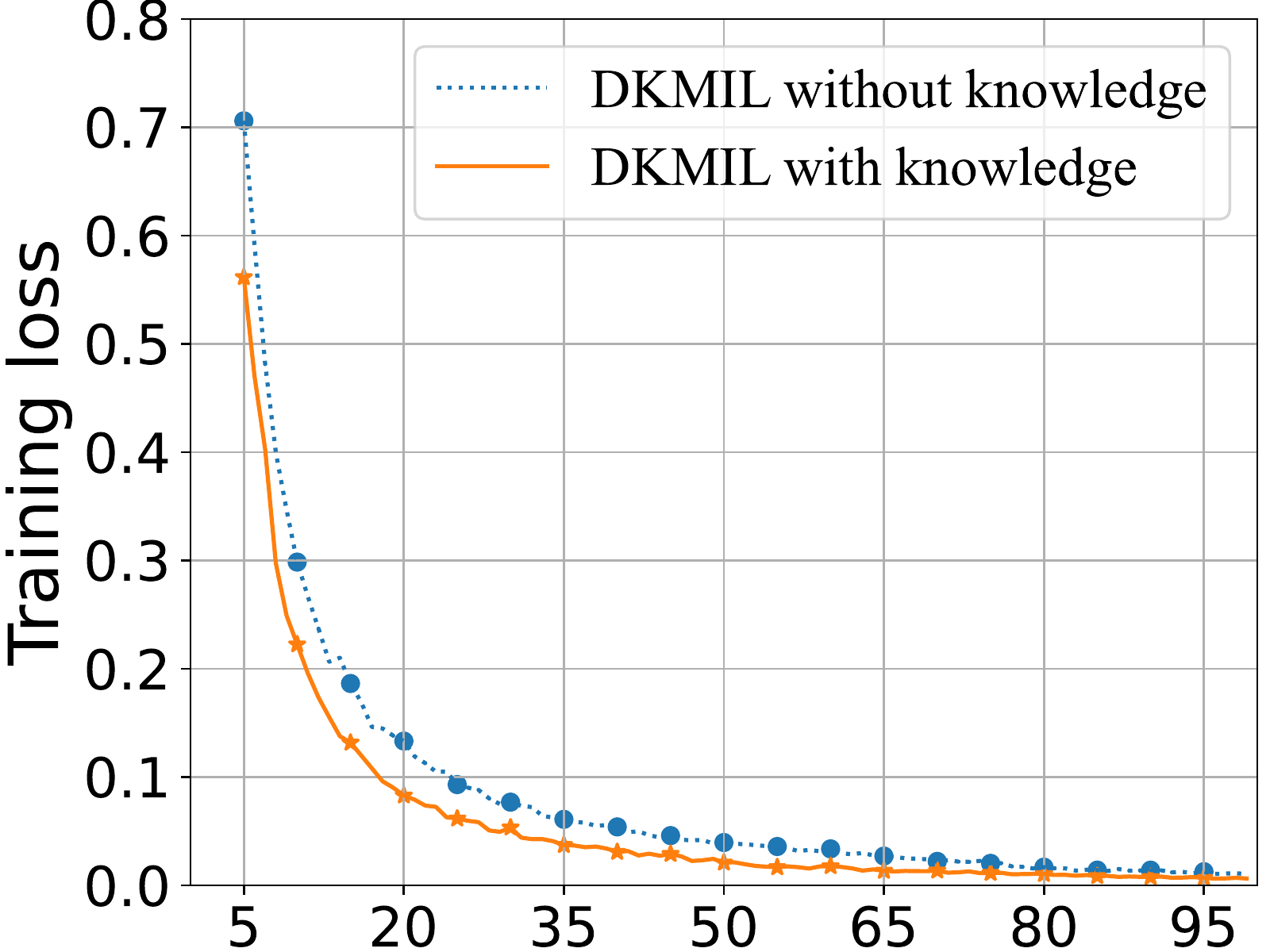}\label{fig: ablation_musk1}}
    \subfloat[Musk2]{\includegraphics[width=0.33\hsize]{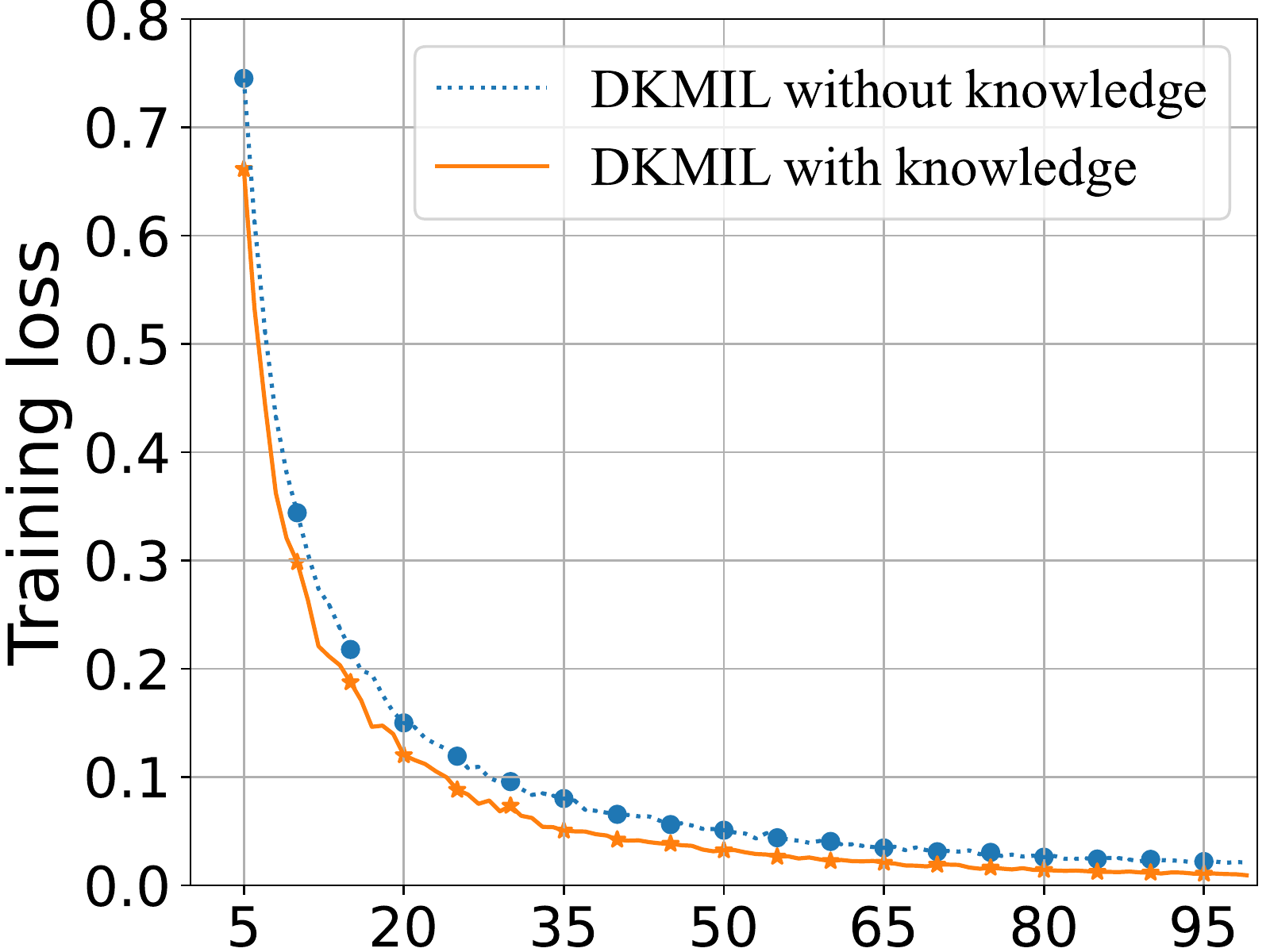}\label{fig: ablation_musk2}}
    \subfloat[Musk1 and Musk2]{\includegraphics[width=0.33\hsize]{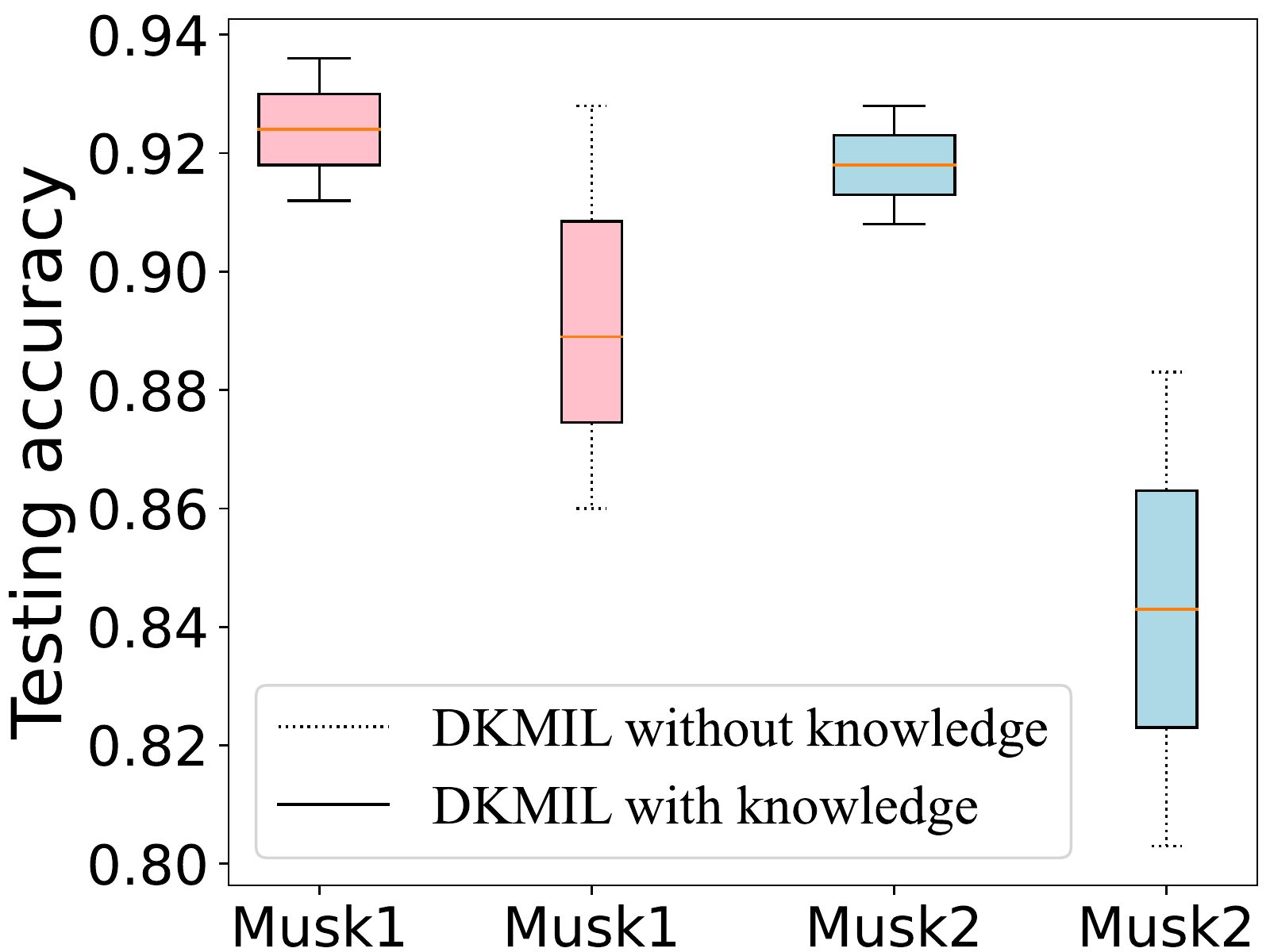}\label{fig: ablation_musk_acc}}
    \caption{
    Ablation study for data-driven knowledge-fusion.
    The data sets used are musk1 and musk2 in the field of drug activity prediction.
    The abscissas of (a) and (b) represent training epochs, while the abscissa of (c) identifies the data set.
    Dashed and solid lines are used to represent the cases without and with knowledge, respectively.
    The center line of the rectangular box in (c) represents the average test accuracy, while the T- and inverted-T-shaped lines indicate the highest and lowest accuracy, respectively.
    }
    \label{fig: ablation_musk}
\end{figure*}

\subsection{Parameter Setups}

\begin{figure}[!htb]
    \centering
    \subfloat[Musk1]{\includegraphics[width=0.99\hsize]{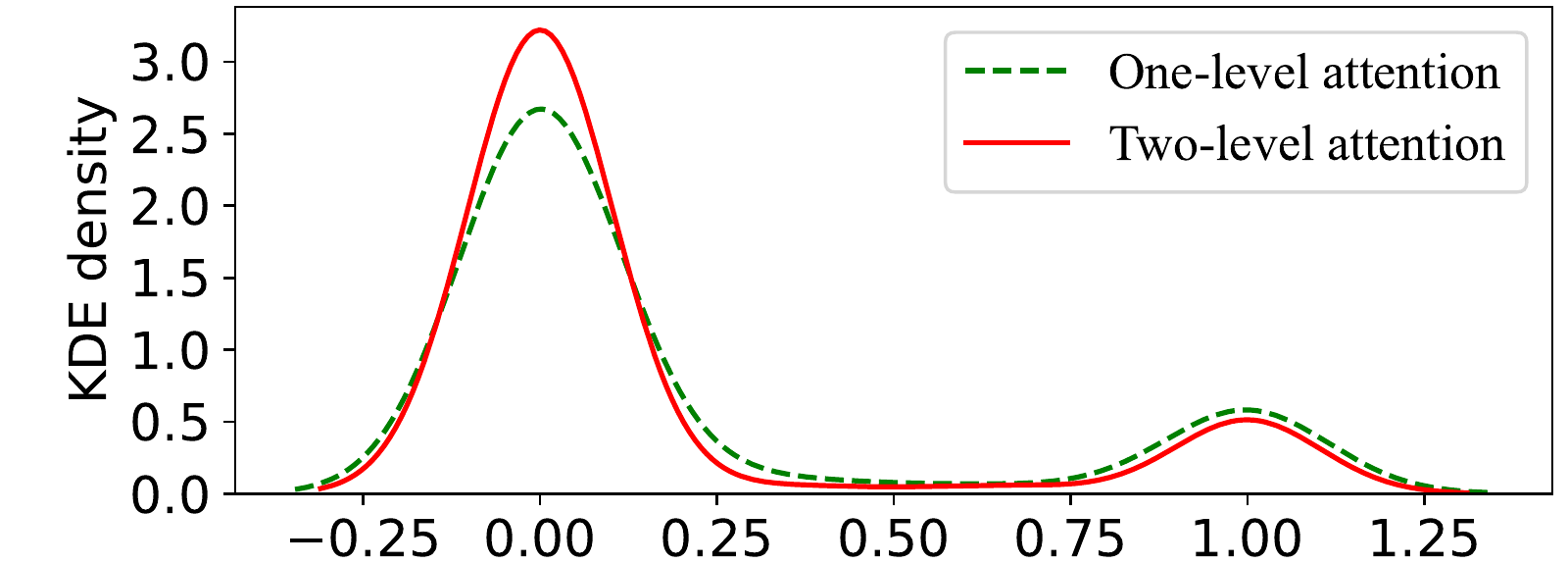}\label{fig: ablation_att_musk1}}

    \subfloat[Musk2]{\includegraphics[width=0.99\hsize]{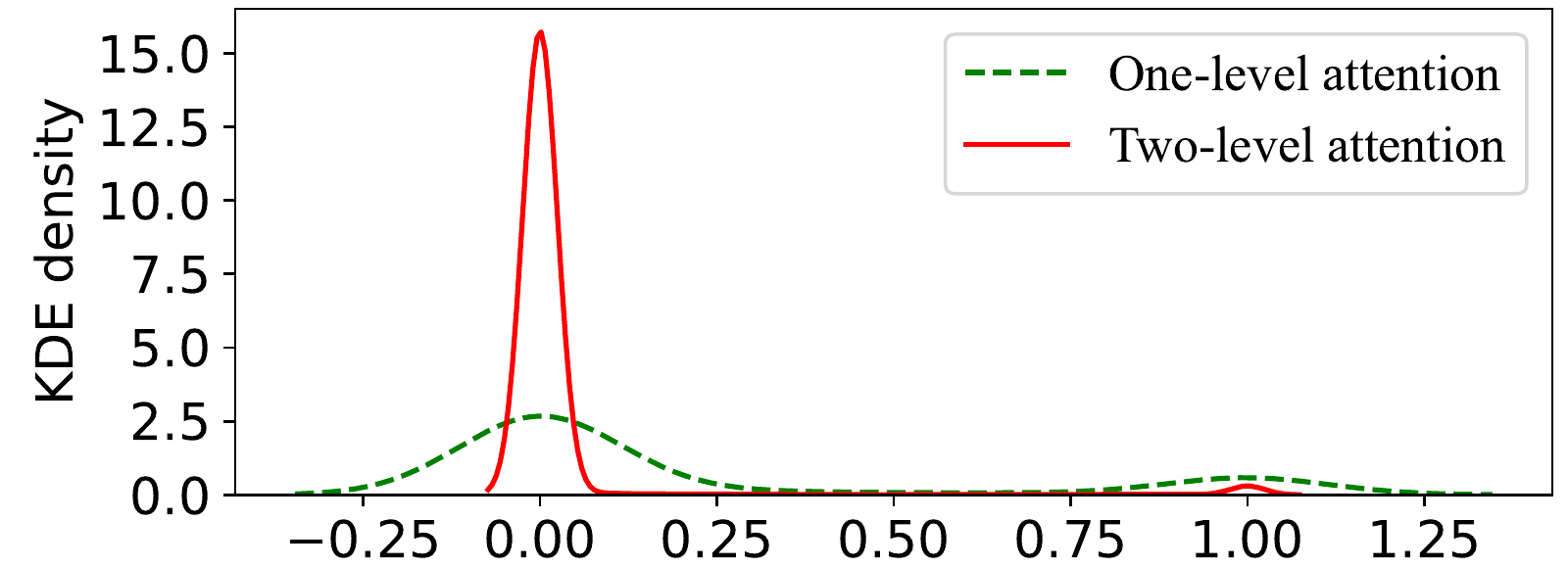}\label{fig: ablation_att_musk2}}
    \caption{
    Ablation study for two-level attention.
    The data sets used are musk1 and musk2 in the field of drug activity prediction.
    Dashed and solid lines represent the cases with one-level and two-level attention, respectively.
    }
    \label{fig: ablation_attention}
\end{figure}

The experiment compares nine algorithms, including four traditional algorithms (miVLAD \cite{Wei:2017:975987}, miFV \cite{Wei:2016:155198}, ELDB \cite{Yang:2021:54565467}, and MSK \cite{Yang:2022:339351}), and five neural network methods (ABMIL \cite{Ilse:2018:21272136}, GAMIL \cite{Ilse:2018:21272136}, LAMIL \cite{Shi:2020:57425745}, DSMIL \cite{Li:2021:1431814328}, and MAMIL \cite{Konstantinov:2022:123}).
These algorithms were selected due to their representation of the latest advancements in MIL or their strong theoretical basis.
Some crucial parameters that require analysis refer to setups of \cite{Yang:2022:109121,Zhang:2023:111}.
The only difference is that we have set the epoch of DSMIL and MAMIL to $100$.
For our designed DKMIL,
the number of random sampling $m_n = 10$,
the size of sampling space
$$
m_s =\left\{
\begin{array}{ll}
10              & N \leq 100;\\
0.1 \times N    & 100 \leq N \leq 500;\\
50              & otherwise,
\end{array}
\right.
$$
the number of key instances and key bags $m_k = 3$,
the minimum dimensions $N_m = 10$,
the mask's ratio parameter $r = 0.1$,
the number of nodes $H_m = 256$ (for mask block),
$H_l = 128$ and $D_l = 64$ (for shallow-level attention),
$H_h = 64$ (for high=level attention)
$H_c = 16$ (for skip connection),
learning rate was set to $5e\text{-}5$.
For the employed optimizer Adam, the weight decay was set to $5e\text{-}5$.
Before each activation function in mask and attention blocks, we add a dropout layer (drop out rate was set to $0.1$) to prevent overfitting.
Additionally, we added $\ell_1$ regularization to the model parameters.
For more details, please see the source code at https://github.com/InkiInki/DKMIL.

We selected five types of MIL data sets, including drug activity prediction, image classification, web recommendation, text classification, video anomaly detection (VAD), and medical diagnosis, to validate DKMIL.
The motivation for this decision will next be explained.

First and foremost, traditional MIL data sets should be taken into account since they represent the origin, development, and exploration of this field.
Therefore, two drug activity (musk1 and musk2) \cite{Dietterich:1997:3171}, three image (elephant, tiger, and fox) \cite{Andrews:2002:561568,Li:2008:9851002}, six web \cite{Zhou:2005:135147}, and $20$ text (news groups) \cite{Zhou:2009:12491256} data sets were used.
Because of how unbalanced their classes are in comparison to other data sets, web1, web2, and web3 were not used in the settings.

Secondly, the algorithm should be able to handle bags with more complex structures, such as those where each image correlates to an instance, in order to meet real-world demands.
As a result of suggestions of \cite{Ilse:2018:21272136,Zhang:2023:111}, $30$ synthetic image data sets based on the three image databases MNIST \cite{Lecun:1998:mnist}, CIFAR10 \cite{Krizhevsky:2009:learning}, and STL10 \cite{Coates:2011:215223} were created.

Furthermore, the structure of MIL data set is well-suited for tasks such as video anomaly detection and medical diagnosis, making it a valuable approach for real-world applications.
For example, videos and 3D magnetic resonance imaging (MRI) scans can be naturally viewed as bags where each instances corresponds to video frames or 2D MRI images.
In this regard, one video anomaly detection (VAD) data set ShanghaiTech \cite{Liu:2018:63566545,Zhong:2019:12371246} was employed to validate the proposed algorithms.
To improve the efficiency of leaning video features, we used the setting of \cite{Tian:2021:49754986}, divided each video into multiple clips and treated each clip as a bag, using I3D \cite{Kay:2017:122} for preprocessing.
We also expanded the application of DKMIL in the area of medical diagnosis using the brain tumor data set \cite{Cheng:2015:0140381}.
Three different patient data types are included in this data set, namely meningioma, glioma, and pituitary.
Its goal is to identify the tumor type based on the patient's 3D magnetic resonance imaging (MRI) scan.

For each data set, the average accuracy and standard deviation (the value with ``$\pm$'') of $5$ times $5$-fold cross validation ($5$CV) were reported.
For MNIST, CIFAR10, and STl10, which already include training and test sets, the average accuracy was determined based on the results of five independent experiments.
Note that accuracy is the only evaluation metric used because the dat aset is class-balanced.

\begin{table*}[!htb]
\caption{
Performance comparison of DKMIL with traditional and neural network-based methods on the benchmark data sets.
Text in bold font indicates the highest classification accuracy for each data set for each row.
``Average'' displays the mean and standard deviation of the algorithm's classification performance across all the listed data sets.
}
\label{table: performance_benchmark}
\centering
\resizebox{\textwidth}{!}{
\begin{tabular}{ccccccccccccc}
\toprule
Data set                      & miVLAD                         & miFV                           & ELDB                           & MSK                            & ABMIL
                              & GAMIL                          & LAMIL                          & DSMIL                          & MAMIL                          & \pmb{DKMIL}\\
\midrule
Musk1                         & $0.847\pm0.011$                & $0.920\pm0.008$                & $0.902\pm0.016$                & $0.860\pm0.013$                & $0.884\pm0.022$
                              & $0.900\pm0.050$                & $0.890\pm0.020$                & $0.913\pm0.021$                & $0.827\pm0.067$                & \pmb{$0.924\pm0.012$}\\
Musk2                         & $0.780\pm0.054$                & $0.890\pm0.020$                & $0.857\pm0.039$                & $0.806\pm0.021$                & $0.822\pm0.017$
                              & $0.863\pm0.042$                & $0.848\pm0.019$                & $0.905\pm0.022$                & $0.814\pm0.061$                & \pmb{$0.918\pm0.010$}\\
Elephant                      & $0.856\pm0.011$                & $0.852\pm0.013$                & $0.843\pm0.012$                & $0.746\pm0.010$                & $0.848\pm0.014$
                              & $0.868\pm0.022$                & $0.872\pm0.005$                & \pmb{$0.898\pm0.010$}          & $0.890\pm0.018$                & $0.897\pm0.009$\\
Tiger                         & $0.843\pm0.008$                & $0.789\pm0.006$                & $0.767\pm0.013$                & $0.734\pm0.016$                & $0.810\pm0.031$
                              & $0.845\pm0.018$                & $0.819\pm0.011$                & $0.851\pm0.015$                & $0.849\pm0.016$                & \pmb{$0.865\pm0.010$}\\
Fox                           & $0.611\pm0.020$                & $0.639\pm0.011$                & $0.648\pm0.014$                & $0.540\pm0.016$                & $0.606\pm0.060$
                              & $0.635\pm0.013$                & $0.561\pm0.022$                & $0.653\pm0.021$                & $0.632\pm0.008$                & \pmb{$0.664\pm0.019$}\\
Web4                          & $0.816\pm0.015$                & $0.807\pm0.812$                & $0.775\pm0.014$                & $0.782\pm0.009$                & $0.844\pm0.027$
                              & $0.845\pm0.035$                & $0.785\pm0.009$                & $0.896\pm0.012$                & $0.845\pm0.011$                & \pmb{$0.923\pm0.005$}\\
Web5                          & $0.821\pm0.015$                & $0.782\pm0.061$                & $0.791\pm0.006$                & $0.775\pm0.010$                & $0.822\pm0.015$
                              & $0.815\pm0.017$                & $0.776\pm0.011$                & $0.858\pm0.050$                & $0.842\pm0.019$                & \pmb{$0.934\pm0.005$}\\
Web6                          & $0.833\pm0.017$                & $0.778\pm0.005$                & $0.778\pm0.008$                & $0.778\pm0.010$                & $0.811\pm0.020$
                              & $0.805\pm0.019$                & $0.782\pm0.005$                & $0.884\pm0.032$                & $0.778\pm0.005$                & \pmb{$0.936\pm0.000$}\\
Web7                          & $0.731\pm0.015$                & $0.687\pm0.030$                & $0.476\pm0.024$                & $0.416\pm0.055$                & $0.713\pm0.021$
                              & $0.698\pm0.025$                & $0.485\pm0.031$                & $0.733\pm0.030$                & $0.531\pm0.065$                & \pmb{$0.811\pm0.009$}\\
Web8                          & $0.746\pm0.019$                & $0.706\pm0.021$                & $0.474\pm0.065$                & $0.518\pm0.030$                & $0.713\pm0.012$
                              & $0.695\pm0.020$                & $0.466\pm0.050$                & $0.753\pm0.023$                & $0.565\pm0.078$                & \pmb{$0.836\pm0.007$}\\
Web9                          & $0.758\pm0.017$                & $0.753\pm0.022$                & $0.420\pm0.035$                & $0.455\pm0.051$                & $0.724\pm0.039$
                              & $0.713\pm0.033$                & $0.503\pm0.021$                & $0.785\pm0.021$                & $0.527\pm0.088$                & \pmb{$0.823\pm0.026$}\\
News.aa                       & $0.836\pm0.027$                & $0.834\pm0.016$                & $0.849\pm0.007$                & $0.854\pm0.005$                & $0.862\pm0.019$
                              & $0.810\pm0.032$                & \pmb{$0.874\pm0.016$}          & $0.872\pm0.016$                & $0.856\pm0.033$                & $0.845\pm0.035$\\
News.cg                       & $0.790\pm0.014$                & $0.802\pm0.008$                & $0.806\pm0.010$                & \pmb{$0.820\pm0.000$}          & $0.609\pm0.015$
                              & $0.610\pm0.017$                & $0.644\pm0.033$                & $0.690\pm0.029$                & $0.706\pm0.042$                & $0.780\pm0.020$\\
News.com                      & $0.702\pm0.042$                & $0.688\pm0.027$                & $0.725\pm0.025$                & $0.738\pm0.008$                & $0.700\pm0.013$
                              & $0.576\pm0.051$                & $0.502\pm0.045$                & $0.682\pm0.018$                & $0.694\pm0.062$                & \pmb{$0.758\pm0.042$}\\
News.csi                      & $0.798\pm0.013$                & $0.637\pm0.021$                & $0.784\pm0.010$                & $0.782\pm0.004$                & $0.744\pm0.029$
                              & $0.662\pm0.013$                & $0.742\pm0.040$                & $0.714\pm0.033$                & $0.730\pm0.020$                & \pmb{$0.803\pm0.037$}\\
News.csm                      & $0.798\pm0.008$                & $0.724\pm0.021$                & $0.817\pm0.028$                & $0.794\pm0.011$                & $0.764\pm0.021$
                              & $0.690\pm0.030$                & $0.800\pm0.018$                & $0.746\pm0.009$                & $0.708\pm0.015$                & \pmb{$0.825\pm0.033$}\\
News.cwx                      & $0.806\pm0.054$                & $0.758\pm0.013$                & $0.789\pm0.018$                & $0.734\pm0.009$                & $0.642\pm0.010$
                              & $0.662\pm0.013$                & $0.588\pm0.033$                & \pmb{$0.832\pm0.019$}          & $0.802\pm0.013$                & $0.810\pm0.008$\\
News.mf                       & $0.716\pm0.029$                & $0.736\pm0.016$                & $0.685\pm0.021$                & $0.700\pm0.010$                & $0.666\pm0.022$
                              & $0.468\pm0.050$                & $0.716\pm0.032$                & $0.730\pm0.024$                & $0.692\pm0.036$                & \pmb{$0.745\pm0.035$}\\
News.ra                       & $0.822\pm0.026$                & $0.718\pm0.025$                & $0.772\pm0.010$                & $0.774\pm0.011$                & $0.706\pm0.008$
                              & $0.698\pm0.023$                & $0.768\pm0.034$                & $0.774\pm0.029$                & $0.748\pm0.022$                & \pmb{$0.825\pm0.010$}\\
News.rm                       & $0.812\pm0.016$                & \pmb{$0.877\pm0.020$}          & $0.798\pm0.017$                & $0.828\pm0.008$                & $0.854\pm0.021$
                              & $0.740\pm0.054$                & $0.871\pm0.026$                & $0.850\pm0.020$                & $0.846\pm0.029$                & $0.768\pm0.032$\\
News.rsb                      & $0.838\pm0.011$                & $0.745\pm0.014$                & $0.834\pm0.012$                & $0.834\pm0.011$                & $0.826\pm0.010$
                              & $0.798\pm0.013$                & \pmb{$0.898\pm0.021$}          & $0.860\pm0.014$                & $0.866\pm0.018$                & $0.870\pm0.014$\\
News.rsh                      & $0.894\pm0.010$                & $0.884\pm0.010$                & $0.834\pm0.016$                & $0.864\pm0.009$                & $0.872\pm0.005$
                              & $0.858\pm0.040$                & $0.920\pm0.021$                & $0.874\pm0.021$                & \pmb{$0.936\pm0.030$}          & $0.935\pm0.019$\\
News.sc                       & $0.818\pm0.023$                & $0.750\pm0.018$                & $0.770\pm0.012$                & $0.764\pm0.009$                & $0.780\pm0.014$
                              & $0.790\pm0.024$                & $0.802\pm0.036$                & \pmb{$0.866\pm0.027$}          & $0.854\pm0.015$                & $0.805\pm0.039$\\
News.se                       & $0.922\pm0.012$                & $0.926\pm0.005$                & \pmb{$0.940\pm0.007$}          & \pmb{$0.940\pm0.000$}          & $0.554\pm0.010$
                              & $0.574\pm0.013$                & $0.572\pm0.036$                & $0.634\pm0.063$                & $0.676\pm0.030$                & $0.903\pm0.038$\\
News.sm                       & $0.804\pm0.023$                & $0.777\pm0.026$                & $0.826\pm0.007$                & $0.836\pm0.005$                & $0.822\pm0.016$
                              & $0.760\pm0.037$                & $0.720\pm0.022$                & $0.868\pm0.019$                & $0.856\pm0.015$                & \pmb{$0.870\pm0.018$}\\
News.src                      & $0.794\pm0.011$                & $0.721\pm0.028$                & \pmb{$0.845\pm0.010$}          & $0.830\pm0.007$                & $0.754\pm0.014$
                              & $0.763\pm0.009$                & $0.820\pm0.028$                & $0.826\pm0.015$                & $0.806\pm0.038$                & $0.780\pm0.029$\\
News.ss                       & $0.850\pm0.012$                & $0.775\pm0.016$                & $0.805\pm0.007$                & $0.798\pm0.004$                & $0.800\pm0.013$
                              & $0.802\pm0.010$                & \pmb{$0.904\pm0.820$}          & $0.868\pm0.013$                & $0.832\pm0.019$                & $0.830\pm0.027$\\
News.tpg                      & $0.820\pm0.019$                & $0.592\pm0.025$                & $0.799\pm0.007$                & $0.798\pm0.008$                & $0.720\pm0.018$
                              & $0.728\pm0.019$                & $0.820\pm0.090$                & \pmb{$0.822\pm0.024$}          & $0.796\pm0.009$                & $0.790\pm0.014$\\
News.tpmid                    & $0.846\pm0.023$                & $0.799\pm0.016$                & $0.827\pm0.005$                & $0.830\pm0.000$                & $0.836\pm0.016$
                              & $0.844\pm0.015$                & $0.844\pm0.012$                & $0.860\pm0.019$                & \pmb{$0.870\pm0.010$}          & \pmb{$0.870\pm0.014$}\\
News.tpmis                    & $0.748\pm0.012$                & $0.752\pm0.015$                & $0.684\pm0.015$                & $0.688\pm0.008$                & $0.720\pm0.013$
                              & $0.711\pm0.032$                & $0.482\pm0.022$                & $0.780\pm0.030$                & \pmb{$0.786\pm0.018$}          & $0.728\pm0.049$\\
News.trm                      & $0.780\pm0.020$                & $0.740\pm0.014$                & $0.717\pm0.011$                & $0.728\pm0.018$                & $0.606\pm0.060$
                              & $0.621\pm0.045$                & $0.514\pm0.064$                & \pmb{$0.796\pm0.021$}          & $0.790\pm0.034$                & $0.765\pm0.037$\\
\midrule
Average                       & $0.801\pm0.020$                & $0.769\pm0.044$                & $0.762\pm0.016$                & $0.753\pm0.012$                & $0.756\pm0.020$
                              & $0.737\pm0.027$                & $0.729\pm0.053$                & $0.809\pm0.023$                & $0.773\pm0.030$                & \pmb{$0.833\pm0.021$}\\
\bottomrule
\end{tabular}}
\end{table*}

\subsection{Ablation Study}

We will evaluate the effectiveness of the two key components of our DKMIL, i.e., the data-driven knowledge fusion and the two-level attention mechanism, in this subsection.

Specifically, the data-driven knowledge fusion module is employed to leverage key information on sample acquisition methods from previous studies and improve the model's learning ability.
Therefore, we compared the model's training loss and test accuracy on the musk1 and musk2 data sets with and without this module to access its effectiveness, as shown in Fig. \ref{fig: ablation_musk}.
To make the training loss curve more smooth and resistant the factors such as random initialization of algorithm parameters, Figs. \ref{fig: ablation_musk1} and \ref{fig: ablation_musk2} display the average results of 25 CV experiments.
The experimental findings suggest that the proposed knowledge fusion module enhances the convergence efficiency, classification performance, and stability of the model.

\begin{table*}[!htb]
\caption{Performance comparison of DKMIL with neural network-based methods on the image classification data sets.
}
\label{table: performance_image}
\centering
\resizebox{0.7\textwidth}{!}{
\begin{tabular}{ccccccc}
\toprule
Data set                      & ABMIL                          & GAMIL                          & LAMIL                          & DSMIL                          & MAMIL                          & \pmb{DKMIL}\\
\midrule
MNIST0                        & $0.992\pm0.011$                & $0.980\pm0.000$                & \pmb{$0.996\pm0.009$}          & $0.968\pm0.011$                & \pmb{$0.996\pm0.009$}          & $0.992\pm0.011$\\
MNIST1                        & $0.928\pm0.105$                & $0.952\pm0.018$                & \pmb{$0.988\pm0.012$}          & $0.968\pm0.018$                & $0.852\pm0.135$                & $0.984\pm0.009$\\
MNIST2                        & $0.948\pm0.018$                & $0.936\pm0.017$                & \pmb{$0.976\pm0.009$}          & $0.972\pm0.018$                & \pmb{$0.976\pm0.009$}          & $0.956\pm0.017$\\
MNIST3                        & \pmb{$0.992\pm0.011$}          & $0.972\pm0.011$                & $0.960\pm0.000$                & $0.968\pm0.023$                & $0.960\pm0.000$                & $0.980\pm0.000$\\
MNIST4                        & $0.944\pm0.017$                & $0.956\pm0.030$                & $0.928\pm0.018$                & $0.968\pm0.033$                & \pmb{$0.996\pm0.009$}          & $0.992\pm0.011$\\
MNIST5                        & $0.956\pm0.009$                & $0.972\pm0.011$                & $0.944\pm0.017$                & $0.976\pm0.009$                & \pmb{$0.980\pm0.000$}          & $0.968\pm0.011$\\
MNIST6                        & \pmb{$1.000\pm0.000$}          & \pmb{$1.000\pm0.000$}          & \pmb{$1.000\pm0.000$}          & \pmb{$1.000\pm0.000$}          & \pmb{$1.000\pm0.000$}          & \pmb{$1.000\pm0.000$}\\
MNIST7                        & $0.976\pm0.009$                & $0.964\pm0.022$                & $0.952\pm0.011$                & $0.936\pm0.026$                & \pmb{$0.980\pm0.000$}          & $0.972\pm0.011$\\
MNIST8                        & $0.944\pm0.022$                & \pmb{$0.980\pm0.014$}          & $0.928\pm0.011$                & $0.904\pm0.022$                & $0.976\pm0.017$                & $0.964\pm0.017$\\
MNIST9                        & $0.952\pm0.011$                & \pmb{$0.972\pm0.018$}          & $0.960\pm0.014$                & $0.956\pm0.017$                & $0.968\pm0.011$                & $0.932\pm0.018$\\
CIFAR0                        & $0.720\pm0.028$                & $0.728\pm0.046$                & $0.752\pm0.027$                & $0.712\pm0.023$                & \pmb{$0.760\pm0.051$}          & $0.720\pm0.035$\\
CIFAR1                        & $0.672\pm0.052$                & $0.704\pm0.061$                & \pmb{$0.716\pm0.046$}          & $0.668\pm0.036$                & $0.632\pm0.048$                & \pmb{$0.716\pm0.039$}\\
CIFAR2                        & $0.640\pm0.028$                & $0.672\pm0.027$                & $0.660\pm0.020$                & $0.652\pm0.011$                & $0.648\pm0.018$                & \pmb{$0.684\pm0.017$}\\
CIFAR3                        & $0.676\pm0.036$                & $0.668\pm0.027$                & $0.684\pm0.017$                & $0.664\pm0.036$                & $0.644\pm0.009$                & \pmb{$0.724\pm0.022$}\\
CIFAR4                        & $0.628\pm0.058$                & $0.660\pm0.020$                & $0.620\pm0.028$                & $0.560\pm0.024$                & $0.592\pm0.048$                & \pmb{$0.664\pm0.017$}\\
CIFAR5                        & $0.692\pm0.046$                & $0.664\pm0.009$                & $0.688\pm0.011$                & $0.704\pm0.033$                & $0.676\pm0.026$                & \pmb{$0.716\pm0.017$}\\
CIFAR6                        & $0.664\pm0.043$                & $0.676\pm0.048$                & $0.652\pm0.030$                & $0.604\pm0.017$                & \pmb{$0.724\pm0.054$}          & $0.716\pm0.022$\\
CIFAR7                        & $0.644\pm0.041$                & $0.624\pm0.017$                & $0.616\pm0.026$                & $0.636\pm0.022$                & $0.664\pm0.061$                & \pmb{$0.728\pm0.036$}\\
CIFAR8                        & $0.696\pm0.038$                & $0.696\pm0.026$                & $0.728\pm0.023$                & $0.688\pm0.033$                & $0.732\pm0.036$                & \pmb{$0.740\pm0.028$}\\
CIFAR9                        & \pmb{$0.700\pm0.035$}          & $0.668\pm0.023$                & $0.632\pm0.027$                & $0.676\pm0.030$                & $0.680\pm0.020$                & $0.644\pm0.026$\\
STl0                          & $0.716\pm0.116$                & $0.824\pm0.017$                & $0.844\pm0.022$                & $0.768\pm0.033$                & $0.828\pm0.030$                & \pmb{$0.868\pm0.018$}\\
STL1                          & \pmb{$0.640\pm0.037$}          & $0.584\pm0.036$                & $0.592\pm0.036$                & $0.588\pm0.033$                & $0.612\pm0.036$                & $0.592\pm0.023$\\
STL2                          & $0.680\pm0.037$                & \pmb{$0.808\pm0.033$}          & $0.796\pm0.033$                & $0.728\pm0.046$                & $0.780\pm0.028$                & $0.804\pm0.033$\\
STL3                          & $0.712\pm0.030$                & $0.684\pm0.009$                & $0.684\pm0.009$                & $0.696\pm0.017$                & $0.700\pm0.020$                & \pmb{$0.740\pm0.025$}\\
STL4                          & \pmb{$0.848\pm0.018$}          & $0.844\pm0.009$                & $0.840\pm0.000$                & $0.844\pm0.009$                & $0.844\pm0.009$                & \pmb{$0.848\pm0.011$}\\
STL5                          & \pmb{$0.668\pm0.050$}          & \pmb{$0.668\pm0.039$}          & $0.628\pm0.011$                & $0.628\pm0.011$                & $0.636\pm0.026$                & $0.664\pm0.022$\\
STL6                          & $0.584\pm0.046$                & $0.568\pm0.033$                & $0.564\pm0.022$                & \pmb{$0.664\pm0.038$}          & $0.544\pm0.017$                & $0.620\pm0.032$\\
STL7                          & $0.576\pm0.026$                & $0.544\pm0.055$                & $0.580\pm0.058$                & $0.564\pm0.014$                & $0.620\pm0.037$                & \pmb{$0.624\pm0.033$}\\
STL8                          & $0.760\pm0.032$                & $0.704\pm0.073$                & $0.780\pm0.028$                & $0.764\pm0.017$                & $0.824\pm0.026$                & \pmb{$0.828\pm0.030$}\\
STL9                          & $0.804\pm0.033$                & \pmb{$0.828\pm0.054$}          & $0.820\pm0.024$                & $0.776\pm0.003$                & $0.808\pm0.090$                & $0.760\pm0.049$\\
\midrule
Average                       & $0.778\pm0.035$                & $0.783\pm0.027$                & $0.784\pm0.020$                & $0.773\pm0.022$                & $0.788\pm0.029$                & \pmb{$0.804\pm0.021$}\\
\bottomrule
\end{tabular}}
\end{table*}

\begin{table*}[!htb]
\caption{Performance comparison of DKMIL with neural network-based methods on the VAD and MRI data sets.}
\label{table: performance_application}
\centering
\resizebox{0.7\textwidth}{!}{
\begin{tabular}{ccccccc}
\toprule
Data set                      & ABMIL                          & GAMIL                          & LAMIL                          & DSMIL                          & MAMIL                          & \pmb{DKMIL}\\
\midrule
Shanghai                      & $0.907\pm0.007$                & $0.907\pm0.005$                & $0.917\pm0.006$                & $0.931\pm0.006$                & $0.866\pm0.016$                & \pmb{$0.936\pm0.040$}\\
Meningioma                    & $0.617\pm0.048$                & $0.580\pm0.051$                & \pmb{$1.000\pm0.000$}          & \pmb{$1.000\pm0.000$}          & $0.583\pm0.047$                & $0.724\pm0.084$\\
Glioma                        & $0.606\pm0.020$                & $0.578\pm0.042$                & \pmb{$1.000\pm0.000$}          & \pmb{$1.000\pm0.000$}          & $0.587\pm0.045$                & $0.719\pm0.038$\\
Pituitary                     & $0.737\pm0.020$                & $0.680\pm0.088$                & \pmb{$1.000\pm0.000$}          & \pmb{$1.000\pm0.000$}          & $0.700\pm0.058$                & $0.847\pm0.083$\\
\bottomrule
\end{tabular}}
\end{table*}

In addition, the inclusion of the knowledge fusion module increases the difficulty of training the model with the key samples.
To address this issue, we introduced a two-level attention module to improve the model's feature extraction ability and enhance its predictability.
To assess the effectiveness of this module, we conducted an ablation study and found that the two-level attention mechanism assigns higher weights to the instances that determine the bag label, while keeping most instances close to 0.
We used kernel density estimation to visualize the distribution of the attention values and confirmed that the two-level attention mechanism is superior to the single-level attention mechanism in assigning instance weights.
Fig. \ref{fig: ablation_attention} provides a visual representation of our findings.

\subsection{Performance Comparison}

Tables \ref{table: performance_benchmark}--\ref{table: performance_application} show the results of performance comparison experiments on benchmark, synthetic, and practical application data sets, respectively.
The experimental results demonstrate that our algorithm performs the best on most data sets, particularly in the areas of drug activity prediction and web recommendation.
For example, DKMIL achieves a classification accuracy of over 90\% on web4 and web5, which is about 10\% higher than the second best performing method.
These advantages are unmatched by comparable methods as they are better equipped to address these concerns.
Furthermore, the average classification performance across all data sets verifies the effectiveness of DKMIL, demonstrating that the data-driven knowledge fusion and two-level attention mechanism are effective at extracting latent knowledge from data sets.
In summary, DKMIL is scalable enough to handle more complex application data sets as well as being able to efficiently handle traditional MIL applications.
To provide a more intuitive demonstration of DKMIL, we plot its attention distribution over three image data sets, encompassing a total of 12 bags, as shown in Fig. \ref{fig: example_image_attention}.
The outcomes indicate that DKMIL has the capability of localizing positive instances well for the well-classified MNIST data set, while assigning almost uniform attention value to all instances in the negative bag.
It is possible that the higher resolution and increased semantic complexity of the A and B data sets make it more difficult for DKMIL to accurately localize positive instances and assign attention values. This could lead to interference and a decrease in overall model performance.

While DKMIL performs well on CIFAR10 and STL10, there is room for improvement on certain data sets, such as MNIST.
This may be due to fixed network parameters not accommodating the higher image resolution and channels of the latter.
Additionally, DKMIL exhibits higher standard deviation in test accuracy on the tumor data set, possibly due to the sizable blank area and lack of professional preprocessing in the original 2D MRI image, which could negatively affect the model's ability to learn.
Further analysis is required for these special cases.
In this regard, the performance of DKMIL on the meningioma data set varies and can be observed through three training accuracy variation curves presented in Figure \ref{fig: performance_tumor}.
The cause of this phenomenon could be due to the algorithm's inability to extract valuable information from the tumor image or the learning rate being insufficient for optimal optimization.
Increasing the number of training epochs may resolve the issue but could lead to heavy hardware loads and contradict the goal of maintaining uniform model parameters.

\begin{figure*}[!htb]
    \centering
    \subfloat[MNIST10]{\includegraphics[width=\hsize]{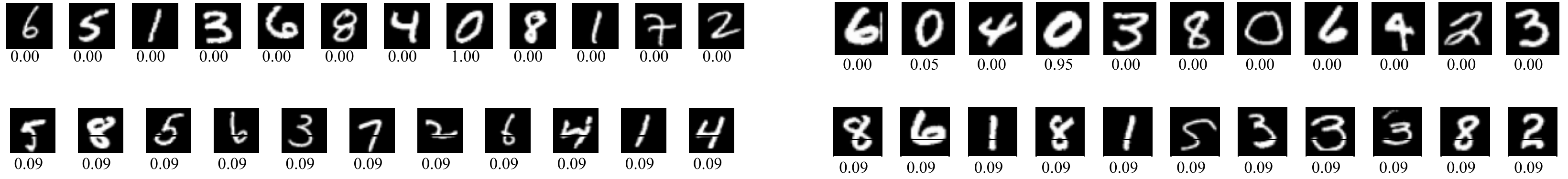}\label{fig: example_mnist}}

    \subfloat[CIFAR10]{\includegraphics[width=\hsize]{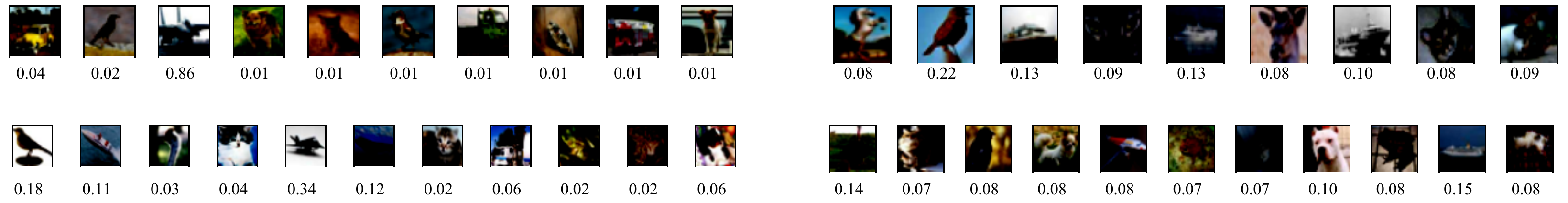}\label{fig: example_cifar10}}

    \subfloat[STL10]{\includegraphics[width=\hsize]{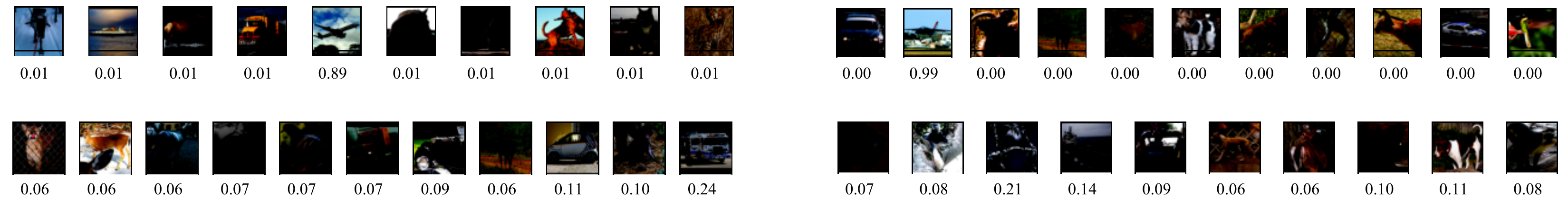}\label{fig: example_stl10}}
    \caption{
    The attention distribution of DKMIL over three image data sets, comprising a total of 12 bags.
    The target class of MNIST0 is zero, while for the other two data sets, the target class is airplane.
    }
    \label{fig: example_image_attention}
\end{figure*}

\begin{table*}[!htb]
\caption{Parameters of MIL neural networks.}
\label{table: parameters}
\centering
\resizebox{0.7\textwidth}{!}{
\begin{tabular}{cccccccc}
\toprule
Data set     & Resolution                 & ABMIL                          & GAMIL                          & LAMIL                          & DSMIL                          & MAMIL                          & \pmb{DKMIL}\\
\midrule
MNIST        & $28 \times 28$             & $4.91e5$                       & $5.55e5$                       & $2.72e5$                       & $1.41e6$                       & $1.49e6$                       & $1.60e5$\\
CIFAR10      & $32 \times 32$             & $7.e7e5$                       & $7.81e5$                       & $3.89e5$                       & $3.32e6$                       & $1.76e6$                       & $1.60e5$\\
STL10        & $96 \times 96$             & $1.11e7$                       & $1.11e7$                       & $5.73e6$                       & $9.75e8$                       & $1.28e7$                       & $2.35e6$\\
Tumor        & $512 \times 512$           & $3.91e8$                       & $3.91e8$                       & $2.01e8$                       & $2.40e9$                       & $3.91e8$                       & $1.57e8$\\
\bottomrule
\end{tabular}}
\end{table*}

Overall, our model achieves superior performance on most data sets, particularly on benchmark data sets, as our average accuracy is several percentage points higher than that of suboptimal algorithms.
Conversely, on the tumor data set, DKMIL demonstrated significantly inferior performance compared to LAMIL and DSMIL. We were surprised by the exceptional performance of these two algorithms. Therefore, we began investigating the model complexity to elucidate the reasons for this phenomenon.
Table 1 displays the parameters of DKMIL and 5 rival algorithms on 4 image data sets with varying resolutions. The findings indicate that our model has significantly fewer network parameters than other neural networks, particularly DSMIL.
However, for the Tumor data set, a parameter scale of such magnitude is insufficient to extract sufficient information for model training.
This observation is also reflected in algorithms such as ABMIL.
It is plausible that LAMIL's exceptional performance can be attributed to the incorporation of multiple optimization objectives to encourage the model to learn key features.
We do not aim to adjust the parameters of DKMIL to make it more compatible with these data sets, as in the prior experiment setup, we fixed all model parameters except for the input layer to ensure its generalizability.

\begin{figure}[!t]
\centering
\includegraphics[width=0.8\linewidth]{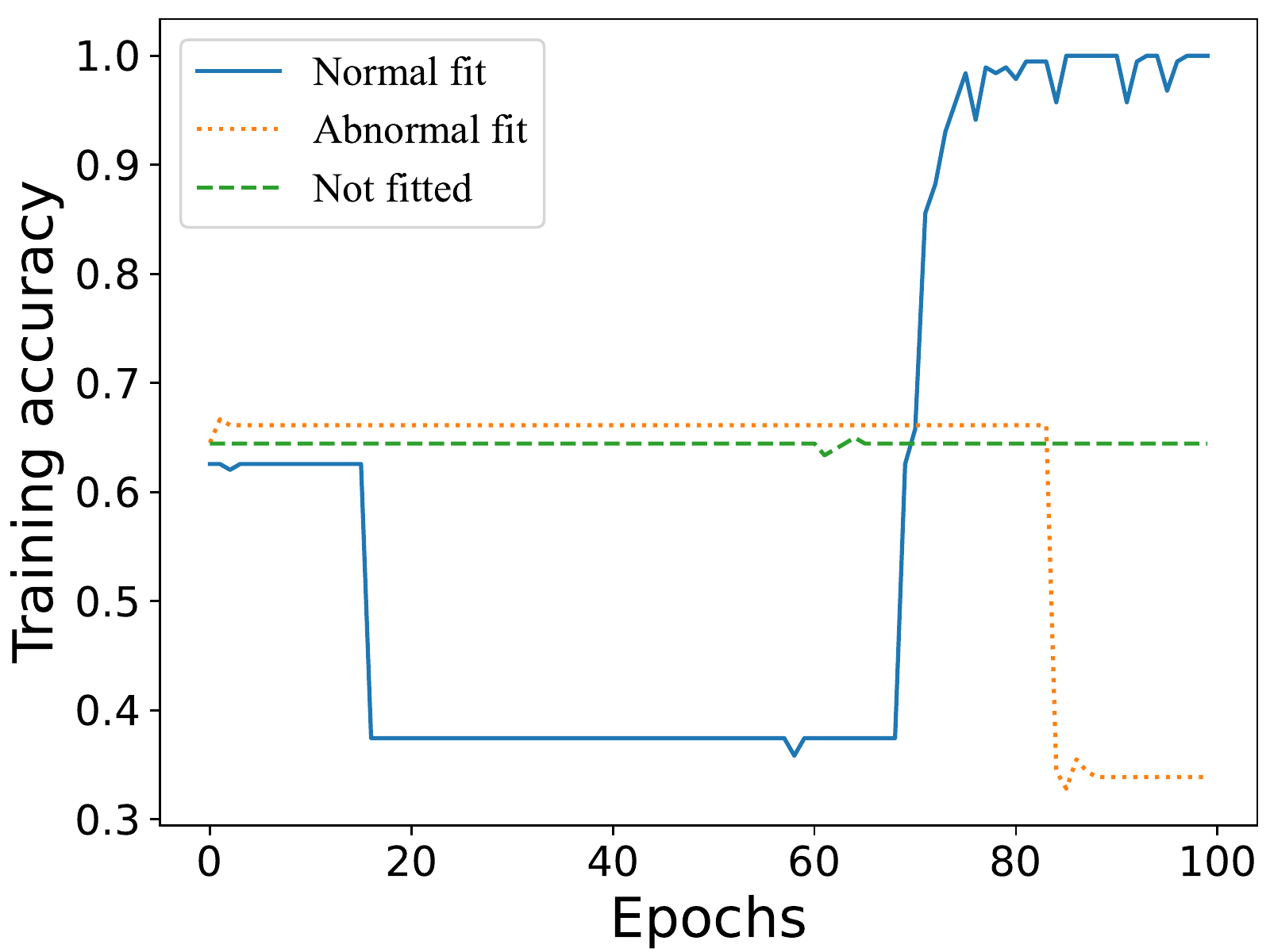}
\caption{
The variation curves of the training accuracy of DKMIL on the meningioma data set.
}
\label{fig: performance_tumor}
\end{figure}

\subsection{Convergence Comparison}

The experiments conducted thus far demonstrate that DKMIL performs exceptionally well on both traditional and synthetic data sets.
However, the algorithm's test accuracy varies significantly when dealing with certain types of data sets, such as MNIST and CIFAR10.
To investigate this phenomenon and also illustrate the differences in learning performance among different deep MIL techniques, we conducted convergence comparison experiments on the MNIST0 and CIFAR0 data sets, which are presented in Fig. \ref{fig: convergence}.
The results show that for the MNIST0 data set, both DKMIL and the other compared algorithms converge effectively, and the test accuracy reaches a stable state in relatively fewer epochs.
Of course, the training accuracy of DKMIL will vary to some amount, this is likely due to the additional processing of knowledge derived from key samples, however, this has a minimal effect on the test results.
For the CIFAR0 data set, all the algorithms show significant overfitting, and the training accuracy might decrease after a certain number of epochs.
Additionally, the test accuracy of the model fluctuates excessively, this might be due to the fact that although the data set is large, it contains too few useful labeled bags.
This implies that we should employ more rational data synthesis techniques and use data sets from more useful applications.

\begin{figure*}[!t]
    \centering
    \subfloat{\includegraphics[width=0.99\hsize]{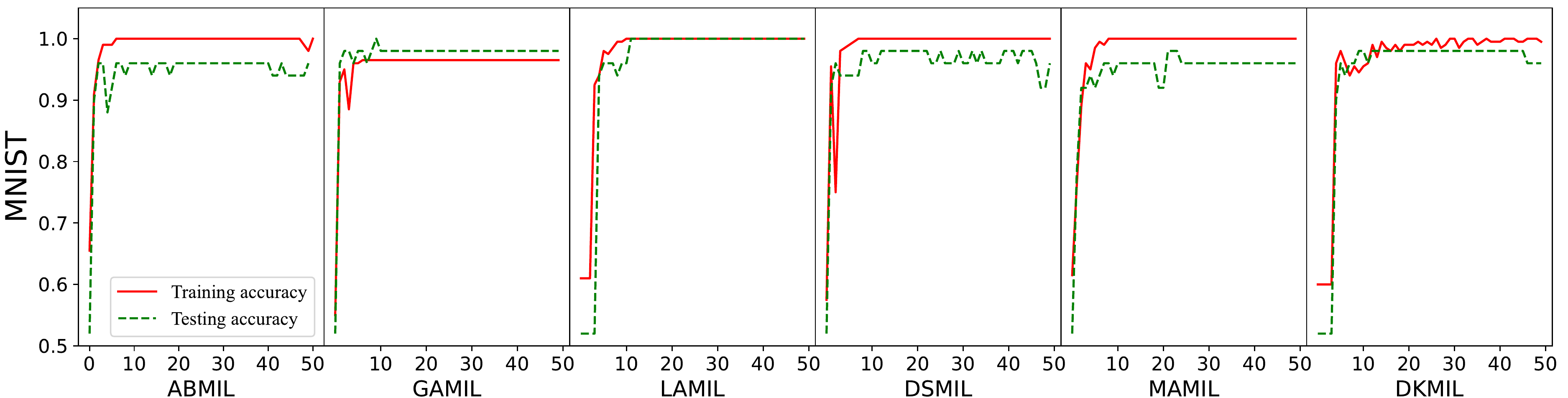}\label{fig: convergence_mnist}}

    \subfloat{\includegraphics[width=0.99\hsize]{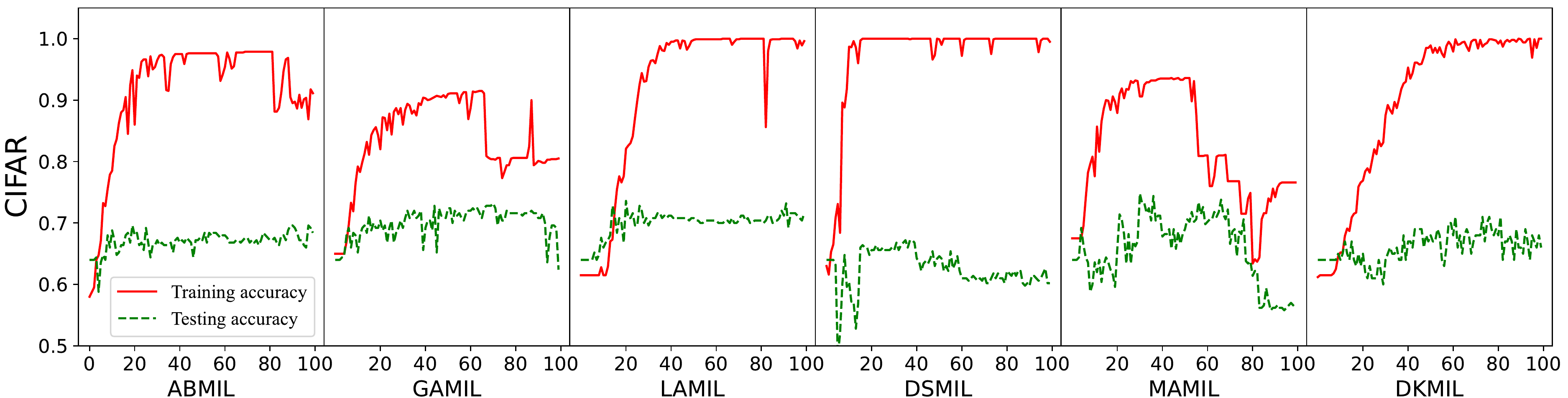}\label{fig: convergence_cifar}}
    \caption{
    Convergence comparison of DKMIL with neural network-based methods on the MNIST0 and CIFAR0 data sets.
    The abscissa represents the training epochs, and the solid and dashed lines represent training accuracy and test accuracy, respectively
    }
    \label{fig: convergence}
\end{figure*}

\begin{table*}[!t]
\caption{
Two-tailed $t$-test results for DKMIL vs. five deep MIL methods on the MNIST and CIFAR10 data sets.}
\label{table: statistical_significance_comparison}
\centering
\resizebox{0.75\textwidth}{!}{
\begin{tabular}{cccccccc}
\toprule
Data sets               & DKMIL-ABMIL     & DKMIL-GAMIL     & DKMIL-LAMIL     & DKMIL-DSMIL     & DKMIL-MAMIL\\
\midrule
MNIST0                  & $3.19e-02$      & $3.56e-05$      & $4.38e-02$      & $7.10e-04$      & $4.56e-05$\\
MNIST1                  & $7.71e-01$      & $5.58e-02$      & $3.47e-01$      & $1.06e-02$      & $3.47e-01$\\
MNIST2                  & $2.71e-01$      & $7.21e-03$      & $5.45e-01$      & $1.11e-01$      & $6.12e-02$\\
MNIST3                  & $4.86e-01$      & $9.55e-02$      & $4.62e-02$      & $1.82e-01$      & $4.62e-02$\\
MNIST4                  & $4.00e-02$      & $1.41e-01$      & $3.94e-11$      & $2.73e-01$      & $3.94e-11$\\
MNIST5                  & $6.72e-04$      & $3.44e-02$      & $1.35e-04$      & $1.66e-01$      & $5.45e-01$\\
MNIST7                  & $3.47e-01$      & $3.47e-01$      & $3.47e-01$      & $3.47e-01$      & $3.47e-01$\\
MNIST8                  & $5.45e-01$      & $4.86e-01$      & $2.03e-02$      & $2.16e-02$      & $1.41e-01$\\
MNIST9                  & $1.43e-01$      & $1.41e-01$      & $3.81e-03$      & $1.25e-03$      & $2.90e-01$\\
CIFAR0                  & $5.11e-02$      & $3.56e-03$      & $1.55e-02$      & $3.75e-02$      & $3.31e-03$\\
CIFAR1                  & $1.00e+00$      & $7.64e-01$      & $1.41e-01$      & $6.78e-01$      & $1.85e-01$\\
CIFAR2                  & $1.67e-01$      & $7.18e-01$      & $1.00e+00$      & $7.71e-02$      & $1.59e-02$\\
CIFAR3                  & $1.72e-02$      & $4.21e-01$      & $7.36e-02$      & $7.21e-03$      & $1.11e-02$\\
CIFAR4                  & $3.37e-02$      & $6.83e-03$      & $1.18e-02$      & $1.27e-02$      & $6.55e-05$\\
CIFAR5                  & $2.17e-01$      & $7.40e-01$      & $1.72e-02$      & $5.05e-05$      & $1.34e-02$\\
CIFAR6                  & $3.05e-01$      & $2.81e-04$      & $1.40e-02$      & $4.88e-01$      & $2.03e-02$\\
CIFAR7                  & $4.36e-02$      & $1.27e-01$      & $5.06e-03$      & $1.73e-05$      & $7.66e-01$\\
CIFAR8                  & $8.96e-03$      & $3.99e-04$      & $5.10e-04$      & $1.27e-03$      & $7.76e-02$\\
CIFAR9                  & $7.33e-02$      & $3.38e-02$      & $4.81e-01$      & $2.91e-02$      & $7.08e-01$\\
\midrule
Accept / Reject         & $12/7$          & $11/8$          & $7/12$          & $8/11$          & $10/9$\\
\bottomrule
\end{tabular}}
\end{table*}

\begin{table*}[!t]
\caption{Vulnerability comparison for DKMIL with neural network-based methods on the MNIST data sets using adversarial perturbation strategy MI-CAP \cite{Zhang:2023:111} ($\xi=0.2$).
$\downarrow$ indicates a decrease in testing accuracy.
}
\label{table: vulnerability comparison}
\centering
\resizebox{0.8\textwidth}{!}{
\begin{tabular}{cccccccc}
\toprule
Data set                & ABMIL                          & GAMIL                          & LAMIL                          & DSMIL                          & MAMIL                          & \pmb{DKMIL}\\
\midrule
MNIST0                  & $0.288\pm0.052$                & $0.252\pm0.062$                & $0.188\pm0.083$                & $0.216\pm0.050$                & $\ \ \ 0.300\pm0.084\downarrow$& $0.224\pm0.038$\\
MNIST1                  & $\ \ \ 0.184\pm0.110\downarrow$& $0.076\pm0.046$                & $0.108\pm0.056$                & $0.124\pm0.062$                & $0.076\pm0.069$                & $0.116\pm0.072$\\
MNIST2                  & $0.392\pm0.131$                & $0.416\pm0.138$                & $\ \ \ 0.420\pm0.111\downarrow$& $0.360\pm0.084$                & $0.312\pm0.039$                & $0.356\pm0.113$\\
MNIST3                  & $0.364\pm0.125$                & $\ \ \ 0.400\pm0.106\downarrow$& $0.316\pm0.063$                & $0.416\pm0.046$                & $0.300\pm0.075$                & $0.288\pm0.026$\\
MNIST4                  & $\ \ \ 0.384\pm0.082\downarrow$& $0.324\pm0.030$                & $0.296\pm0.076$                & $0.356\pm0.041$                & $0.228\pm0.030$                & $0.316\pm0.082$\\
MNIST5                  & $0.268\pm0.097$                & $\ \ \ 0.320\pm0.017\downarrow$& $0.224\pm0.082$                & $0.264\pm0.062$                & $0.272\pm0.046$                & $0.304\pm0.081$\\
MNIST6                  & $\ \ \ 0.352\pm0.023\downarrow$& $0.320\pm0.244$                & $0.204\pm0.017$                & $0.212\pm0.046$                & $0.196\pm0.099$                & $0.264\pm0.050$\\
MNIST7                  & $0.188\pm0.041$                & $0.260\pm0.033$                & $0.104\pm0.103$                & $\ \ \ 0.300\pm0.054\downarrow$& $0.200\pm0.011$                & $0.240\pm0.064$\\
MNIST8                  & $0.356\pm0.092$                & $0.412\pm0.099$                & $0.196\pm0.124$                & $\ \ \ 0.480\pm0.073\downarrow$& $0.336\pm0.067$                & $0.248\pm0.076$\\
MNIST9                  & $0.224\pm0.065$                & $\ \ \ 0.496\pm0.086\downarrow$& $0.364\pm0.071$                & $0.492\pm0.075$                & $0.312\pm0.130$                & $0.254\pm0.045$\\
\midrule
Average                 & $0.300\pm0.082$                & $\ \ \ 0.328\pm0.086\downarrow$& $0.242\pm0.079$                & $0.322\pm0.059$                & $0.253\pm0.065$                & $0.261\pm0.065$\\
\bottomrule
\end{tabular}}
\end{table*}

\subsection{Statistical Significance Comparison}

We have conducted a thorough analysis of the comparison between the performance and learning abilities of our algorithm and those of other algorithms.
These findings confirm the effectiveness of our DKMIL.
However, the unique data-driven knowledge fusion that we have developed sets it apart from other algorithms, which raises the question of whether this remains true from a statistical perspective.
Table \ref{table: statistical_significance_comparison} provides a summary of $p$-value of two-tailed $t$-test between DKMIL and all comparative algorithms in this experimental setup.
All paired $t$-test values are calculated using a $95\%$ confidence level ($\alpha = 0.05$).
According to statistical theory, if the $p$-value is greater than $0.05$, there is no significant difference between the two algorithms.
The results indicate that DKMIL is statistically similar to the rival algorithms.
However, the fundamental concept of DKMIL is distinct from the compared methods as it integrates data-driven knowledge fusion to investigate more intelligent models.
This also demonstrates from another perspective that our method is viable and can be used as an alternative to other techniques.

\subsection{Vulnerability Comparison}

Recent research on the security of the MIL algorithms also offer new perspective for evaluating our algorithm \cite{Zhang:2023:111}.
This reminds us of the importance of enhancing the robustness of the algorithm against adversarial examples.
Although we currently do not take this scenario into account, it is important to expose this problem through experiments in an intuitive manner.
To this end, we conducted experiments using the most aggressive MI-CAP attack with the ``att'' mode.
As a result, as shown in Table \ref{table: vulnerability comparison}, reveal that most algorithms are vulnerable to attacks, leading to reduced testing accuracy and varied predictions of varying degrees, thereby compromising the model's reliability.
Therefore, going forward, we will focus on addressing this issue.

\section{Conclusion}

Our primary goal in writing this article was to deepen our comprehension of intelligence and explore novel methods of integrating it into the MIL algorithm.
To this end, we introduced data-driven knowledge fusion as an initial exploration of this idea and utilized this module to create a robustly scalable interface between the key samples and the model, thereby facilitating model training.
We have theoretically proven the scalability of our approach, while its effectiveness has been demonstrated through multiple experiments.
Although our algorithm has achieved good experimental results, including the best performance on more than half of the data sets, there is still room for improvement and further exploration, and we will focus on addressing this in our future work:
\begin{enumerate}
  \item
  The data-driven knowledge fusion module serves as an effective interface between key samples and models with strong theoretical scalability.
  However, since it is still in its initial exploratory stage, the prior knowledge incorporated into it is limited.
  Thus, it is imperative to extend the module to enable more efficient learning and enhance its adaptability to complex learning environments.
  \item
  Although our algorithm, DKMIL, demonstrates strong learning ability in the experiments, it currently struggles to handle and process certain data sets, such as CIFAR10 and tumor, due to overfitting issues.
  To address this, further investigation into the characteristics of these data sets and modifications to the model architecture may be necessary.
  \item
  DKMIL does not currently incorporate any security considerations.
  However, if the algorithm were to be implemented on data sets such as VAD, it would be essential to integrate adversarial example handling into the processing module.
\end{enumerate}


\end{document}